\documentclass[10pt,journal,compsoc]{IEEEtran}
\ifCLASSOPTIONcompsoc
  \usepackage[nocompress]{cite}
\else
  \usepackage{cite}
\fi
\ifCLASSINFOpdf
   \usepackage[pdftex]{graphicx}
\else
\fi
\usepackage{amsmath}
\usepackage{array}
\ifCLASSOPTIONcompsoc
 \usepackage[caption=false,font=footnotesize,labelfont=sf,textfont=sf]{subfig}
\else
 \usepackage[caption=false,font=footnotesize]{subfig}
\fi
\usepackage{url}
\usepackage{amsmath,amsfonts,bm}

\def\eqref#1{equation~\ref{#1}}
\def\1{\bm{1}}

\DeclareMathAlphabet{\mathsfit}{\encodingdefault}{\sfdefault}{m}{sl}
\SetMathAlphabet{\mathsfit}{bold}{\encodingdefault}{\sfdefault}{bx}{n}

\DeclareMathOperator*{\argmax}{arg\,max}
\DeclareMathOperator*{\argmin}{arg\,min}

\def\##1\#{\begin{align}#1\end{align}}
\def\$#1\${\begin{align*}#1\end{align*}}

\def\given{\,|\,}
\def\biggiven{\,\big{|}\,}
\def\Biggiven{\,\Big{|}\,}

\def\tr{\mathop{\text{tr}}\kern.2ex}

\long\def\comment#1{}

\def\tr{\mathop{\text{tr}}}

\def\cS{{\mathcal{S}}}

\newcommand{\lan}{\langle}
\newcommand{\ran}{\rangle}

\newcommand{\cA}{\mathcal{A}}
\newcommand{\cB}{\mathcal{B}}

\newcommand{\cD}{\mathcal{D}}

\newcommand{\cJ}{\mathcal{J}}

\newcommand{\cL}{\mathcal{L}}
\newcommand{\cM}{\mathcal{M}}

\newcommand{\cP}{\mathcal{P}}

\newcommand{\EE}{\mathbb{E}}

\newcommand{\RR}{\mathbb{R}}

\newcommand{\kl}{\text{KL}}
\newcommand{\var}{\text{Var}}
\newcommand{\regr}{\text{SubOptGap}}

\newcommand{\perr}{\varepsilon_{\text{Pess}}}
\usepackage[colorlinks,
linkcolor=red,
anchorcolor=blue,
citecolor=blue
]{hyperref}

\usepackage{times}
\usepackage{float}
\usepackage{wrapfig}
\usepackage{booktabs}       % professional-quality tables
\usepackage[english]{babel}
\usepackage{algorithm}
\usepackage{algpseudocode}
\usepackage{multirow}

\usepackage[normalem]{ulem} % for \sout{}

\usepackage{amsmath}
\usepackage{amssymb}
\usepackage{mathtools}
\usepackage{amsthm}
\theoremstyle{definition}
\newtheorem{definition}{Definition}[section]
\newtheorem{theorem}{Theorem}[section]

\newtheorem{lemma}[theorem]{Lemma}
\newtheorem{assumption}[theorem]{Assumption}

\makeatletter

\newcommand{\Rmnum}[1]{\expandafter\@slowromancap\romannumeral #1@}
\makeatother
\usepackage{color}

\begin{document}
\title{False Correlation Reduction for Offline Reinforcement Learning}

\author{Zhihong~Deng, 
    Zuyue~Fu, 
    Lingxiao~Wang, 
    Zhuoran~Yang,\\ 
    Chenjia~Bai, 
    Tianyi~Zhou, 
    Zhaoran~Wang, 
    Jing~Jiang%
% Michael~Shell,~\IEEEmembership{Member,~IEEE,}
%         John~Doe,~\IEEEmembership{Fellow,~OSA,}
%         and~Jane~Doe,~\IEEEmembership{Life~Fellow,~IEEE}% <-this % stops a space
\IEEEcompsocitemizethanks{\IEEEcompsocthanksitem Z. Deng, J. Jiang are with Centre for Artificial Intelligence, FEIT, University of Technology Sydney, NSW 2007, Australia (Email: zhi-hong.deng@student.uts.edu.au, jing.jiang@uts.edu.au).\protect\\
\IEEEcompsocthanksitem Z. Fu, L. Wang. Z. Wang are with the Department of Industrial Engineering and Management Sciences, Northwestern University, Evanston, IL 60208, USA (zuyuefu2022@u.northwestern.edu, lingxiaowang2022@u.northwestern.edu,  zhaoran-wang@gmail.com).\protect\\
\IEEEcompsocthanksitem Z. Yang is with the Department of Statistics and Data Science, Yale University, 24 Hillhouse Avenue, New Haven, CT 06511, USA (zryang1993@gmail.com).\protect\\
\IEEEcompsocthanksitem C. Bai is with Shanghai Artificial Intelligence Laboratory, Shanghai 200232, China (e-mail: baichenjia@pjlab.org.cn).\protect\\
\IEEEcompsocthanksitem T. Zhou is with the Department of Computer Science and UMIACS, University of Maryland, College Park, MD 20742, USA (Email: tianyi.david.zhou@gmail.com).}%
% \IEEEcompsocitemizethanks{\IEEEcompsocthanksitem M. Shell was with the Department
% of Electrical and Computer Engineering, Georgia Institute of Technology, Atlanta,
% GA, 30332.\protect\\
% % note need leading \protect in front of \\ to get a newline within \thanks as
% % \\ is fragile and will error, could use \hfil\break instead.
% E-mail: see http://www.michaelshell.org/contact.html
% \IEEEcompsocthanksitem J. Doe and J. Doe are with Anonymous University.}% <-this % stops an unwanted space
\thanks{Manuscript received November xx, 2022; revised xx xx, xxxx.}}


\markboth{Journal of \LaTeX\ Class Files,~Vol.~14, No.~8, August~2015}%
{Shell \MakeLowercase{\textit{et al.}}: Bare Demo of IEEEtran.cls for Computer Society Journals}



\IEEEtitleabstractindextext{%
\begin{abstract}
Offline reinforcement learning (RL) harnesses the power of massive datasets for resolving sequential decision problems. Most existing papers only discuss defending against out-of-distribution (OOD) actions while we investigate a broader issue, the false correlations between epistemic uncertainty and decision-making, an essential factor that causes suboptimality. In this paper, we propose falSe COrrelation REduction (SCORE) for offline RL, a practically effective and theoretically provable algorithm. We empirically show that SCORE achieves the SoTA performance with 3.1x acceleration on various tasks in a standard benchmark (D4RL). The proposed algorithm introduces an annealing behavior cloning regularizer to help produce a high-quality estimation of uncertainty which is critical for eliminating false correlations from suboptimality. Theoretically, we justify the rationality of the proposed method and prove its convergence to the optimal policy with a sublinear rate under mild assumptions.
\end{abstract}

}

\maketitle

\IEEEdisplaynontitleabstractindextext

% \fi
\IEEEpeerreviewmaketitle

% =======================================================
%
%                    Introduction
% 
% =======================================================
\IEEEraisesectionheading{\section{Introduction}\label{sec:introduction}}

\IEEEPARstart{O}{ffline} reinforcement learning (RL) aims to learn the optimal policy from a pre-collected dataset without interacting with the environment. Theoretically, experience replay allows for the direct application of off-policy RL algorithms to this setting, but they perform poorly in practice~\cite{fujimotoOffpolicyDeepReinforcement2019,fuD4RLDatasetsDeep2021}. Many research works attribute this problem to out-of-distribution (OOD) actions~\cite{fujimotoOffpolicyDeepReinforcement2019,wuBehaviorRegularizedOffline2019a,kumarStabilizingOffpolicyQlearning2019}. Because the offline setting disallows learning by trial and error, an agent may "exploit" OOD actions to attack the value estimator, resulting in a highly suboptimal learned policy. Although this provides an intuitive explanation, the relationship between OOD actions and the suboptimality of the learned policy is ambiguous and lacks a mathematical description.

As opposed to this, false correlation is a rigorously defined concept by mathematically decomposing suboptimality. Due to the insufficient data coverage, a correlation exhibits between epistemic uncertainty and decision-making. As a result, the agent is biased towards suboptimal policies that look good only by chance. Such false correlations can be generated not only by OOD actions but also by insufficient coverage over the state space (or, equivalently, OOD states). Moreover, in-distribution samples have differences in uncertainty, and those with higher uncertainty may subtly raise suboptimality when the agent greedily pursues the maximum estimated value.

\begin{figure}
    \includegraphics[width=1.0\linewidth]{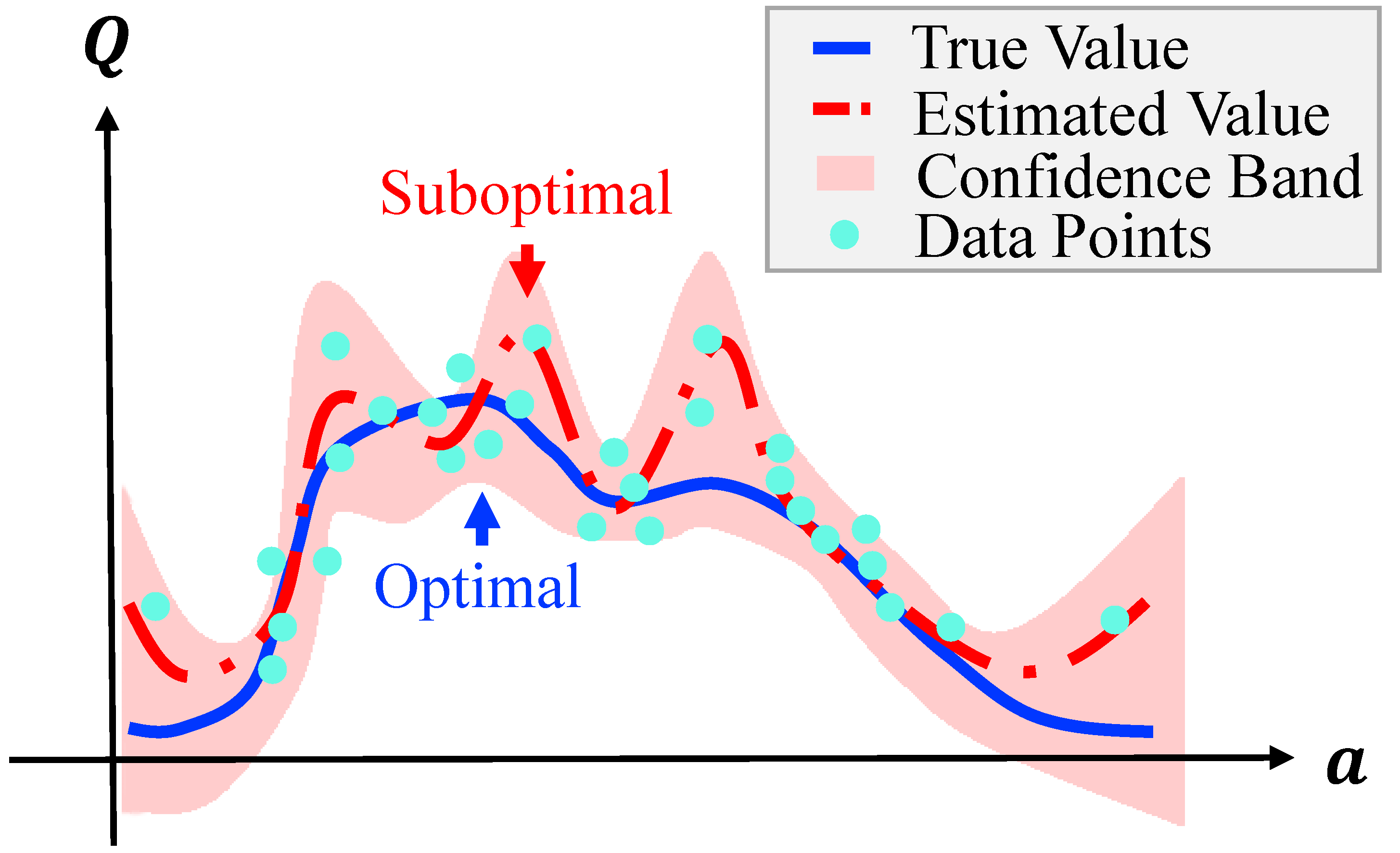}
    \caption{An example of false correlation: the epistemic uncertainty is correlated with the value, making a suboptimal action with high uncertainty appear to be better than the optimal one.}
    \label{fig:example}
\end{figure}
According to recent theoretical studies, pessimism in the face of uncertainty can solve this problem~\cite{jinPessimismProvablyEfficient2021a,xieBellmanconsistentPessimismOffline2021a}. An example of false correlation is shown in \figurename~\ref{fig:example}, where the estimated value function encourages the agent to pick a suboptimal action while the lower confidence bound (pessimism) recovers the optimal one. Unfortunately, this approach fails in practice~\cite{fujimotoOffpolicyDeepReinforcement2019,levineOfflineReinforcementLearning2020a} due to the inability to obtain high-quality estimations of uncertainty.

In this paper, we propose falSe COrrelation REduction (SCORE) for offline RL, a practically effective and theoretically provable algorithm. Specifically, SCORE introduces an annealing behavioral cloning regularizer on top of pessimism in the face of uncertainty. This regularizer drives the agent to concentrate on the dataset distribution during the initial stages of training, producing high-quality uncertainty estimates. In the later phases, SCORE gradually decays the regularization weight towards zero, avoiding the bias towards the behavioral policy. 
On the theoretical side, we generalize the conclusion in~\cite{jinPessimismProvablyEfficient2021a} to the context of infinite-horizon regularized MDP and incorporate the policy optimization process, showing SCORE converges to the optimal policy at a sublinear rate. Importantly, this result does not rely on strong assumptions regarding access to the exact value of OOD samples~\cite{baiPessimisticBootstrappingUncertaintyDriven2022} or sufficient coverage over the sample space~\cite{kumarConservativeQLearningOffline2020,yuMOPOModelbasedOffline2020a}. Furthermore, SCORE eliminates the need to sample and calculate the target value of OOD data, which distinguishes it from the approach presented in~\cite{baiPessimisticBootstrappingUncertaintyDriven2022}.

Through extensive experiments on the D4RL benchmark, we demonstrate that SCORE is not only theoretically sound but also obtains promising empirical results across various data settings. Furthermore, the simplicity of SCORE brings approximately a 3.1x speedup over the previous SoTA method.

% =======================================================
%
%                    Preliminaries
% 
% =======================================================
\section{Preliminaries}
\label{sec:preliminaries}
In this section, we first formalize the offline RL problem, and analyze the performance gap between the optimality policy and a learned policy. Then we show how pessimism can eliminate false correlations from the suboptimality.

\subsection{The Offline RL Problem}
Consider a MDP $\mathcal{M}=(\mathcal{S}, \mathcal{A}, P, R, \gamma, d_0)$, where $\mathcal{S}$ and $\mathcal{A}$ represent the state space and the action space respectively. $P\colon \cS \times\cS \times \cA \to [0,1]$ is the Markov transition function, $R\colon \cS \times \cA \to \RR$ is the reward function, $\gamma \in (0,1)$ is the discount factor, and $d_0\colon \cS \to [0,1]$ is the initial distribution of states. In offline RL, the agent is given a static dataset $\mathcal{D}=\{(s^{(i)}, a^{(i)}, s'^{(i)}, r^{(i)})\}_{i=1}^{N}$ collected by the behavioral policy  $\pi_\beta$. Suppose that $d^{\pi}(s,a)$ denotes the discounted state-action distribution of a policy $\pi$. We have $(s^{(i)}, a^{(i)}) \sim d^{\pi_\beta}(\cdot, \cdot), s'^{(i)}\sim P(\cdot \given s^{(i)},a^{(i)})$, and $r^{(i)} = R(s^{(i)}, a^{(i)})$. Then the goal of offline RL is to search for a policy $\pi \colon \cA \times \cS \to [0,1]$ that maximizes the expected cumulative reward $\cJ(\pi) = \EE_{\pi}[ \sum_{t = 0}^\infty \gamma^t \cdot R(s_t, a_t)]$ given a static dataset $\cD$, where expectation $\EE_\pi[\cdot ]$ is taken with respect to $s_0 \sim d_0(\cdot )$, $a_t \sim \pi(\cdot \given s_t)$, and $s_{t+1} \sim P(\cdot \given s_t, a_t)$. With a slight abuse of notation, we refer to $\mathcal{D}$ as the dataset distribution.

\subsection{Suboptimality Decomposition}
In offline RL, the samples are drawn from a fixed distribution $\mathcal{D}$ instead of the environment. Therefore, in value iteration, the standard Bellman optimality operator $\mathcal{B}$ gets replaced by its empirical counterpart $\widehat{\mathcal{B}}$ where the transition probabilities and rewards are estimated by the sample average in $\cD$. Since the dataset only covers partial information of the environment, the agent would be learning with bias. In this paper, we formalize such bias for any action-value function $Q\colon \cS \times \cA \to \RR$ as follows:
\begin{equation}
\label{eq:epistemic_error}
\iota(s,a) = \mathcal{B}Q(s,a) - \widehat{\mathcal{B}}Q(s,a). 
\end{equation}
Since $\iota(s, a)$ characterizes the error arising from insufficient information about the environment in knowledge and gradually converges to zero as we learn more about the state-action pair $(s, a)$ (including its long-term effects), we refer to it as the \emph{epistemic error}. In the ideal case, the dataset accurately mirrors the environment, i.e., $\widehat{\mathcal{B}}=\mathcal{B}$, resulting in zero epistemic error. The agent can learn the optimal policy offline just like in the online setting. However, this is almost impossible in real-world domains. In general, the dataset only contains limited information and the epistemic error persists throughout the learning process.

According to~\cite{jinPessimismProvablyEfficient2021a}, the suboptimality of a learned policy $\widehat{\pi}$, i.e., the performance gap between $\widehat{\pi}$ and the optimal policy $\pi^*$, can be decomposed mathematically into three different components:
\begin{equation}
\label{eq:suboptimality}
\begin{aligned}
&\operatorname{SubOpt}(\widehat{\pi} ; s_{0})
= V^{\pi^{*}}(s_{0})-V^{\widehat{\pi}}(s_{0})\\
&=\underbrace{-\sum_{t = 0}^\infty \gamma^t\mathbb{E}_{\widehat{\pi}}\left[\iota(s_{t}, a_{t}) \given s_{0}\right]}_{(\mathrm{i})\text{: False Correlation}} + \underbrace{\sum_{t = 0}^\infty \gamma^t \mathbb{E}_{\pi^*}\left[\iota(s_{t}, a_{t}) \given s_{0}\right]}_{(\mathrm{ii})\text{: Intrinsic Uncertainty}}\\
&+ \underbrace{\sum_{t = 0}^\infty \gamma^t \mathbb{E}_{\pi^*}\left[\left\langle\widehat{Q}\left(s_{t}, \cdot\right), \pi^*\left(\cdot \given s_{t}\right)-\widehat{\pi}\left(\cdot \given s_{t}\right)\right\rangle_{\mathcal{A}} \biggiven s_{0}\right]}_{(\mathrm{iii})\text{: Optimization Error}},
\end{aligned}
\end{equation}
where $s_0$ is the initial state, $\widehat Q$ is an estimated Q function, and $V^{\pi}(s) = \langle Q^{\pi}(s, \cdot), \pi(\cdot \given s) \rangle$ is the state-value function. It is straightforward that the suboptimality of the optimal policy $\pi^*$ is zero and lower indicates a better policy. Term ($\mathrm{ii}$) in~\eqref{eq:suboptimality} is proved to arise from the information-theoretic lower bound and thus is impossible to eliminate. Meanwhile, term ($\mathrm{iii}$) is non-positive as long as the policy $\widehat{\pi}$ is greedy with respect to the estimated action-value function $\widehat{Q}$. Therefore, controlling term ($\mathrm{i}$) is the key to reducing suboptimality in offline RL. Next, we explain how pessimism could help address this issue.

\subsection{Pessimism}
\label{subsec:pessimism}
Let $\widehat{Q}\colon \mathcal{S} \times \mathcal{A} \rightarrow \mathbb{R}$ represents an action value estimator based on the dataset. We first define an uncertainty quantifier $U$ with confidence $\xi \in (0,1)$ as follows.

\begin{definition}[$\xi$-Uncertainty Quantifier]
\label{def:uncertainty}
$U: \mathcal{S} \times \mathcal{A} \rightarrow \mathbb{R}$ is a $\xi$-uncertainty quantifier with respect to the dataset distribution $\mathcal{D}$ if the event
\begin{equation}
\mathcal{E}=\left\{ |\widehat{\mathcal{B}}\widehat{Q}(s, a) - \mathcal{B}\widehat{Q}(s, a)| \leq U(s, a), \forall (s,a)\in \mathcal{S}\times\mathcal{A} \right\}
\end{equation}
satisfies $\text{Pr}(\mathcal{E}|\mathcal{D}) \geq 1-\xi$.
\end{definition}

In Definition~\ref{def:uncertainty}, $U$ measures the uncertainty arising from approximating $\mathcal{B}\widehat{Q}$ with $\widehat{\mathcal{B}}\widehat{Q}$, where $\cB$ is the true Bellman optimality operator and $\widehat{\cB}$ is the empirical Bellman operator. We remark that $\widehat{\cB}$ can be constructed implicitly by treating $\widehat{\cB}\widehat{Q}\colon \cS \times \cA \to \RR$ as a whole. 
When $\mathcal{B}\widehat{Q}$ and $\widehat{\mathcal{B}}\widehat{Q}$ differ by a large amount, $U$ should be large, while when the two quantities are sufficiently close, $U$ can be very small or even zero. Based on this definition, Jin et al.~\cite{jinPessimismProvablyEfficient2021a} construct a pessimistic Bellman operator as follows:
\begin{equation}
\label{eq:bellman}
\widehat{\mathcal{B}}^{-}\widehat{Q}(s,a) := \widehat{\mathcal{B}}\widehat{Q}(s,a) - U(s, a). 
\end{equation}
According to Definition~\ref{def:uncertainty}, $\widehat{\mathcal{B}}^{-}\widehat{Q}(s,a) \leq \mathcal{B}\widehat{Q}(s,a)$ holds for all state-action pairs with a high probability, i.e., the Q-value derived from~\eqref{eq:bellman} lower bounds the one derived from the standard Bellman operator $\mathcal{B}$. In other words, \eqref{eq:bellman} provides a pessimistic estimation. 
Replacing the empirical Bellman operator $\widehat{\mathcal{B}}$ in~\eqref{eq:epistemic_error} with the pessimistic Bellman operator, it holds that:
\begin{equation}
\label{eq:suboptimality_bound}
0 \leq \iota(s,a) = \mathcal{B}\widehat{Q}(s,a) - \widehat{\mathcal{B}}^{-}\widehat{Q}(s,a) \leq 2U(s,a).
\end{equation}
Since the epistemic error $\iota(s,a)$ is non-negative in this case, term ($\mathrm{i}$) in~\eqref{eq:suboptimality} only reduces the suboptimality. As a result, pessimism eliminates false correlations. Meanwhile, the suboptimality is now upper-bounded by $\sum_{t = 0}^\infty 2\gamma^t \mathbb{E}_{\pi^*}\left[U(s,a) \given s_0\right]$, so the remaining issue is finding a sufficiently small $\xi$-uncertainty quantifier that satisfies Definition~\ref{def:uncertainty}.

% =======================================================
% 
%                    Our Method 
% 
% =======================================================
\section{False Correlation Reduction for Offline RL}
\label{sec:score}
In this section, we begin with a motivating example in Section~\ref{subsec:example} to demonstrate the universality of false correlations. We then introduce a practical algorithm named SCORE in Section~\ref{subsec:algorithm}. In Section~\ref{subsec:theory}, we further analyze the theoretical properties of the proposed algorithm.

\subsection{A Motivating Example}
\label{subsec:example}
Consider an autonomous vehicle navigating through a complex urban environment, aiming to optimize its route for efficient travel. 
During data collection, the vehicle successfully reaches a specific location during rush hour, receiving a high return for the fast completion. 
However, this success was merely due to a lucky sequence of green lights rather than an optimal policy. 
When an RL agent learns from the dataset and strives to maximize return, it may be prone to bias towards this suboptimal route, leading to inefficient and frustrating experiences in real-world deployment. 
This example illustrates how false correlations can significantly influence offline learning and highlights the critical need to address this issue for practical success.
% Removing the undesirable correlations of uncertainty arising from inadequate information is essential when learning offline.

\subsection{False Correlation Reduction}
\label{subsec:algorithm}
% As shown in Section~\ref{subsec:pessimism}, pessimism eliminates false correlations in offline RL. What remains is to design a proper uncertainty quantifier. 

Based on Definition~\ref{def:uncertainty} and~\eqref{eq:suboptimality_bound}, it is straightforward that $U(s, a)=|\widehat{\mathcal{B}}\widehat{Q}(s, a) - \mathcal{B}\widehat{Q}(s, a)|$ achieves the tightest bound, i.e., $U$ accurately portrays the epistemic uncertainty arising from approximating $\mathcal{B}$ using $\widehat{\mathcal{B}}$. Note that the Bellman operator acts on the value function, so $U$ must consider the input state-action pair's missing information in predicting its long-term effects. 
Besides, since the state and action spaces are enormous in real-world domains, using value function approximators (such as deep neural networks) is necessary. In this case, we can neither derive an analytical solution of the epistemic uncertainty like in linear MDPs~\cite{jinPessimismProvablyEfficient2021a} nor evaluate it by maintaining a counter or a pseudo-counter for all state-action pairs.

% \begin{wrapfigure}{r}{0.65\textwidth}
% \begin{minipage}{0.63\textwidth}
\begin{algorithm}[htp]
\caption{False Correlation Reduction for offline RL}
\label{alg:score}
\begin{algorithmic}[1]
\State Initialize critic networks $\{Q_{\theta_i}\}_{i=1}^{M}$ and actor network $\pi_{\phi}$, with random parameters $\{\theta_i\}_{i=1}^{M}, \phi$
\State Initialize target networks $\{\theta'_{i}\}_{i=1}^{M} \leftarrow \{\theta_i\}_{i=1}^{M}, \phi' \leftarrow \phi$
\State Initialize replay buffer with the dataset $\mathcal{D}$
\For{$t = 1 \text{ to } T$}
    \State Sample $n$ transitions $(s,a,s',r)$ from $\mathcal{D}$
    \State $a' \leftarrow \pi_{\phi'}(s')+\mathbf{\epsilon}, \mathbf{\epsilon} \sim \operatorname{clip}(\mathcal{N}(0, \sigma^2), -c, c)$
    \For{$i=1 \text{ to } M$}
        \State Update $\theta_i$ to minimize~\eqref{eq:q_objective}. \Comment{Pessimism}
    \EndFor
    \If{$t \% d=0$}
        \State Update $\phi$ to maximize~\eqref{eq:pi_objective}.
        \State Update target networks: 
        \State $ \theta'_{i} \leftarrow \tau \theta'_{i} + (1-\tau) \theta_{i},\ \phi'\leftarrow \tau \phi' + (1-\tau) \phi$.
    \EndIf
    \If{$t \% d_{\text{bc}}=0$}
        \State $\lambda=\gamma_{\text{bc}} \cdot \lambda$
    \EndIf
\EndFor
\end{algorithmic}
\end{algorithm}
% \end{minipage}
% \end{wrapfigure}
Estimating epistemic uncertainty with deep neural networks is an important research topic. One of the most popular approaches is the bootstrapped ensemble method~\cite{osbandDeepExplorationBootstrapped2016,lakshminarayananSimpleScalablePredictive2017}. Each ensemble member is trained on a different version of data generated by a bootstrap sampling procedure. This approach provides a general and non-parametric way to approximate the Bayesian posterior distribution, so the standard deviation of multiple estimations can be regarded as a reasonable estimation of the uncertainty. Previous works mainly use uncertainty as a bonus to promote efficient exploration in online RL ~\cite{osbandDeepExplorationBootstrapped2016,osbandRandomizedPriorFunctions2018, ciosekBetterExplorationOptimistic2019, leeSUNRISESimpleUnified2021}, while we utilize uncertainty as a penalty to reduce false correlations. Next, we include this technique into generalized policy iteration to develop a complete RL algorithm.

\noindent\textbf{Policy Evaluation.} 
To implement the pessimistic Bellman operator practically, we maintain $M$ independent critics $\{Q_{\theta_i}\}_{i=1}^{M}$ and their corresponding target networks $\{Q_{\theta'_i}\}_{i=1}^{M}$ in the policy evaluation step. Samples are drawn uniformly from the offline dataset, and the learning objective of each critic $Q_{\theta_i}$ is as follows,
\begin{equation}
\label{eq:q_objective}
\begin{aligned}
\mathcal{L}(Q_{\theta_i}) = \mathbb{E}_{s,a,s',r \sim \mathcal{D}, a' \sim \pi(\cdot \given s')}\left[ \left(Q_{\theta_i}(s,a)-y_i\right)^2 \right],\\
y_{i} = r+ \gamma Q_{\theta'_{i}}(s', a') - \beta u(s, a),
%u(s,a) &= \sqrt{\frac{1}{M} \sum_{i=1}^{M} \left(Q_{\theta_i}(s,a) - \bar{Q_\theta}(s,a)\right)^2}.
\end{aligned}
\end{equation}
where $\beta$ is a hyperparameter that controls the strength of the uncertainty penalty, and $u(\cdot,\cdot)$ is computed as the standard deviation of  $\{Q_{\theta_i}\}_{i=1}^{M}$.
From the Bayesian perspective, the output of the critic ensemble forms a distribution over the Q-value that approximates its posterior. As a result, the standard deviation quantifies the uncertainty in beliefs, i.e., \eqref{eq:q_objective} implements the $\xi$-uncertainty quantifier $U$ in~\eqref{eq:bellman} using the bootstrapped ensemble method.
Remark that, in the theoretical algorithm for linear MDPs~ \cite{jinPessimismProvablyEfficient2021a}, the uncertainty quantifier and the penalty coefficient are in closed forms. All samples participate in the calculation simultaneously and are used only once, which is almost impossible to implement, especially when dealing with massive datasets collected from complex dynamics. To make the proposed method compatible with large data sets, we optimize the Q networks using mini-batch gradient descent. The samples are used repeatedly in varying order. The uncertainty estimator is unstable at the early stage as it is derived from networks trained from scratch, so penalizing Q with a small quantity (controlled by $\beta$) is preferable. While most existing approaches use the smallest Q-value as the target value to avoid overestimation, \eqref{eq:q_objective} updates each critic $Q_{\theta_i}$ toward its corresponding target network $Q_{\theta'_i}$. As a result, \eqref{eq:q_objective} guarantees temporal consistency and passes the uncertainty over time~\cite{osbandDeepExplorationBootstrapped2016,osbandRandomizedPriorFunctions2018}.

\noindent\textbf{Policy Improvement.} 
Designing an uncertainty-based algorithm for offline RL is a non-trivial task. We attribute the failure of previous work~\cite{fujimotoOffpolicyDeepReinforcement2019,levineOfflineReinforcementLearning2020a,yuCOMBOConservativeOffline2021} to the inability to obtain high-quality uncertainty estimations, which, as pointed out in Section~\ref{subsec:pessimism}, is the key to addressing the false correlation issue. Empirically, the bootstrapped ensemble method alone cannot produce such estimation. To this end, we propose using the following objective function for policy $\pi_\phi$ in the policy improvement step:
\begin{equation}
\label{eq:pi_objective}
\begin{aligned}
\mathcal{L}(\pi_\phi) = \mathbb{E}_{s, a \sim \mathcal{D}}\left[\min_{i}Q_{\theta_i}(s, \pi_\phi( s)) - \lambda \Vert \pi_\phi(s)-a \Vert^{2}_{2}\right].
\end{aligned}
\end{equation}
The behavior cloning loss $\Vert \pi_\phi(s)-a \Vert^{2}_{2}$ functions as a regularization term. While the ensemble networks initially struggle to accurately quantify epistemic uncertainty, this term guides the policy to stay near the dataset distribution. 
It preventss the agent from "exploiting" OOD data to attack the critics. Meanwhile, it helps the critics to be sufficiently familiar with the dataset distribution to deliver high-quality uncertainty estimations.
In particular, we gradually decrease the regularization coefficient $\lambda$ during the training process. The regularizer provides a good initialization at the beginning. Later in the training process, the regularization effect becomes weaker and weaker, and the pessimistic Q-values gradually dominate the policy objective. This way, SCORE returns to a pure uncertainty-based method that reliably implements the pessimism principle, avoiding the bias towards the behavioral policy. 
Alternatively, we can understand this regularizer from the optimization perspective~\cite{guoBatchReinforcementLearning2022}. Directly maximizing the value function is challenging, and behavior cloning lowers the difficulty of the optimization problem at the early stage. As the training process proceeds, the regularization effect decreases, so the objective function gradually returns to the original problem, i.e., maximizing the pessimistic value function. 
The complete algorithm is summarized in Algorithm~\ref{alg:score}.
% which frees the algorithm from explicitly modeling the behavioral policy $\pi_\beta$~\cite{fujimotoOffpolicyDeepReinforcement2019,kumarStabilizingOffpolicyQlearning2019,wuUncertaintyWeightedActorCritic2021a}.  

\subsection{Theoretical Results} \label{subsec:theory}
In this section, we theoretically analyze the proposed algorithm and show that it achives a sublinear rate of convergence. We begin with the notion of regularized MDP. 

\noindent\textbf{Regularized MDP.}
For any behavior policy $\pi_0$, based on the definition of the MDP $\cM = (\cS, \cA, P, R, \gamma, d_0)$, we introduce its regularized counterpart $\cM_{\lambda} = (\cS, \cA, P, R, \gamma, d_0, \lambda)$, where $\lambda$ is the regularization parameter. Specifically, for any policy $\pi$ in $\cM_{\lambda}$, the regularized state-value function $V_\lambda^\pi$ and the regularized action-value function $Q_\lambda^\pi$ are defined as 
\$ %\label{eq:soft-val}
\begin{aligned}
&V_\lambda^\pi(s) = \EE_{\pi}\Bigl[\sum_{t = 0}^\infty \gamma^t \cdot \bigl( r(s_t, a_t)\\
&\qquad\qquad - \lambda\cdot \log\bigl({\pi(\cdot\given s_t)}/{\pi_0(\cdot \given s_t)}\bigr) \bigr)\Biggiven s_0 = s\Bigr],\\
&Q_\lambda^\pi(s,a) = r(s,a) + \gamma \cdot \EE_{s'\sim P(\cdot \given s, a)} \bigl[V_\lambda^\pi(s')\bigr],\\
&\qquad\qquad\qquad\qquad\qquad\qquad\qquad \forall (s,a) \in \cS\times \cA,
\end{aligned}
\$
respectively. We remark that such a regularization term in the definition of $V_\lambda^\pi$ functions as a behavior cloning term. Throughout the learning process, we anneal the regularization parameter $\lambda$ so that the impact of the behavior cloning term gradually decreases. Formally, for a collection of regularized MDPs $\{\cM_{\lambda_k}\}_{k = 0}^K$, we aim to minimize the suboptimality gap defined as follows,
\#\label{eq:regr}
\regr(K) = \min_{k \in \{0, 1, \ldots, K-1\}} \bigl (V_k^*(s_0) - V_{k}^{\pi_k}(s_0) \bigr ).
\#
Here we denote by $V_k^* = V_{\lambda_k}^{\pi_k^*}$ and $V_{k}^{\pi_k} = V_{\lambda_k}^{\pi_k}$ for notational convenience, where $\pi^*_k \in \argmax_\pi \EE_{s_0 \sim d_0} [ V_{\lambda_k}^\pi(s_0)]$ is an optimal policy for $\cM_{\lambda_k}$. In other words, equation~\ref{eq:regr} measures the suboptimality gap between the best policy $\pi_{k^*}$ and the corresponding optimal policy $\pi_{k^*}^*$ under the regularized MDP $\cM_{\lambda_{k^*}}$, where $k^* = \argmin_{k \in \{0, 1, \ldots, K-1\}}  (V_k^*(s_0) - V_{k}^{\pi_k}(s_0))$. 

\noindent\textbf{Theo-SCORE.} 
For the simplicity of discussion, we introduce a theoretical counterpart of Algorithm~\ref{alg:score} named Theo-SCORE. 
At the $k$-th iteration of Theo-SCORE, with the estimated pessimistic Q-function $Q_k$, the Theo-SCORE objective for the regularized MDP $\cM_{\lambda_k}$ is defined as follows, 
\#\label{eq:dpg-obj}
\cL^k_{\rm SCORE}(\pi) = \EE_{s\sim \cD} \bigl[\lan Q_k(s, \cdot), \pi(\cdot \given s) \ran\notag\\
\hfill - \lambda_k \cdot \kl (\pi(\cdot \given s) \| \pi_0(\cdot \given s) ) \bigr],
\#
where $\cD$ is the static dataset, and the KL divergence corresponds to the behavior cloning term. In the policy improvement step of Theo-SCORE, we employ deterministic policy gradient \cite{silverDeterministicPolicyGradient2014} to maximize \eqref{eq:dpg-obj}. We remark that the objective function in \eqref{eq:dpg-obj} is equivalent to \eqref{eq:pi_objective} under Gaussian policies with the same covariance. While in the policy evaluation step of Theo-SCORE, we assume there exists an oracle that uses the $\xi$-uncertainty quantifier $U(s,a)$ defined in Definition \ref{def:uncertainty} to construct a pessimistic estimator of the Q-function, which is  practically achieved by \eqref{eq:q_objective}, as shown in Section \ref{subsec:algorithm}. 
% Thus, Theo-SCORE described above is indeed equivalent to its practical counterpart SCORE. 

To better study the convergence of Theo-SCORE, we further introduce a helper algorithm termed offline proximal optimization (OPO). Formally, we consider the linear function parameterization in the $k$-th iteration as follows, 
\#\label{eq:lin-param}
\begin{aligned}
\pi_{\phi_k} \propto \exp( f_{\phi_k}(s,a)),\\
f_{\phi_k}(s,a) = \psi(s,a)^\top \phi_k,\\
Q_k(s,a) = \theta_k(s)^\top a,
\end{aligned}
\#
where $\psi$ and $\theta_k$ are feature vectors, and $f_{\phi_k}$ is the energy function. 
We denote by $\pi_k = \pi_{\phi_k}$ and $f_k = f_{\phi_k}$ for notational convenience. 
With pessimistic Q-function $Q_k$ and current policy $\pi_k$ in the $k$-th iteration, OPO's objective function for the regularized MDP $\cM_{\lambda_k}$ is defined as follows, 
\#\label{eq:ppo-obj}
\cL^k_{\rm OPO}(\phi) = \EE_{s\sim \cD} \Bigl[ \Bigl \lan Q_k(s, \cdot) - \lambda_k \cdot \log \frac{\pi_\phi(\cdot\given s)}{\pi_0(\cdot\given s)}, \pi_{\phi}(\cdot \given s)  \Bigr \ran\notag\\
\hfill - \eta_k\cdot \kl \bigl (\pi_\phi (\cdot\given s) \| \pi_k(\cdot \given s)\bigr ) \Bigr], 
\#
where $\pi_0$ is the behavior policy and $\eta_k$ is the regularization parameter.

\noindent\textbf{Equivalence between Theo-SCORE and OPO.} 
Under the linear function parameterization in~\eqref{eq:lin-param}, we now relate the objective functions in~\eqref{eq:ppo-obj} and~\eqref{eq:dpg-obj} in the following lemma. 
For the sake of simplicity, we define $I_\phi = \EE_{s\sim \cD}[I_\phi(s)]$, where $I_\phi (s)= \var_{a\sim \pi_\phi(\cdot\given s)}[\psi(s,a)]$. 

\begin{lemma}[Equivalence between Theo-SCORE and OPO]\label{lemma:dpg=ppo}
The stationary point $\phi_{k+1}$ of $\cL^k_{\rm OPO}(\phi)$ satisfies
\$
\begin{aligned}
\phi_{k+1} = &\frac{\eta_k \phi_k + \lambda_k \phi_0}{\eta_k + \lambda_k} + (\eta_k + \lambda_k)^{-1} \cdot I_{\phi_{k+1}}^{-1}\\
&\cdot \EE_{s\sim \cD} \bigl[ \nabla_a Q_k(s, \Pi_{\phi_{k+1}}(s)) \nabla_\phi \Pi_{\phi_{k+1}}(s) \bigr], 
\end{aligned}
\$
where $\Pi_\phi(s) = \EE_{a\sim \pi_\phi(\cdot \given s)}[a]$ is the deterministic policy associated with $\pi_\phi$. 
\end{lemma}
% \begin{proof}
% See Section \ref{sec:dpg=ppo} for a detailed proof. 
% \end{proof}

Lemma~\ref{lemma:dpg=ppo} states that maximizing the offline proximal objective is equivalent to an implicit natural policy gradient step corresponding to the maximization of the Theo-SCORE objective. As a result, to study the convergence of Theo-SCORE, it suffices to analyze pessimistic OPO. 

\noindent\textbf{Convergence Analysis.}  
For simplicity of presentation, we take the regularization parameter $\lambda_k = \alpha^k$, where $0 < \alpha < 1$ quantifies the speed of annealing.
Recall that we employ pessimism to construct estimated Q-functions $Q_k$ at each iteration $k$, which ensures that there exists a $\xi$-uncertainty quantifier $U(s,a)$ defined in Definition~\ref{def:uncertainty}. Formally, we apply the following assumption on the estimated Q-functions, which can be achieved by a bootstrapped ensemble method as shown in Section~\ref{subsec:algorithm}. 

\begin{assumption}[Pessimistic Q-Functions]\label{ass:pess}
For any $k\in [K]$, $U\colon \cS\times \cA \to \RR$ is a $\xi$-uncertainty quantifier for the estimated Q-function $Q_k$, i.e., the event 
\$
\begin{aligned}
\mathcal{E}_K= \bigl\{ |\widehat{\mathcal{B}} Q_k(s, a) - \mathcal{B} Q_k(s, a)| \leq U(s, a),\\
\hfill \forall (s,a,k)\in \mathcal{S}\times\mathcal{A}\times[K] \bigr \}
\end{aligned}
\$
holds with probability at least $1 - \xi$. 
\end{assumption}

We further define the pessimistic error as follows, 
\#\label{eq:perr}
\perr = \sum_{t = 0}^\infty 2\gamma^t \cdot \EE_{\pi^*}\bigl[ U (s_t, a_t) \given s_0 \bigr].
\#
Such a pessimistic error in \eqref{eq:perr} quantifies the irremovable intrinsic uncertainty \cite{jinPessimismProvablyEfficient2021a}. Now, we introduce our main theoretical result as follows. 

\begin{theorem}\label{thm:main}
Suppose $\lambda_k = \alpha^k$ and $\eta_k + \lambda_k = \sqrt{\zeta / K}$ for any $k \geq 0$, where 
\$
\begin{aligned}
\zeta =  \bigl (1 + \alpha^4(1 - &\alpha)^{-4} \bigr)^2\\
&\cdot \sum_{t = 0}^\infty \gamma^t \cdot \EE_{\pi^*} \bigl[ \kl \bigl( \pi^*(\cdot \given s_t) \| \pi_0(\cdot \given s_t) \bigr) \given s_0 \bigr], 
\end{aligned}
\$
then for OPO with estimated Q-functions satisfying Assumption \ref{ass:pess}, it holds with probability at least $1 - \xi$ that
\$
\regr(K) = O\bigl((1 - \gamma)^{-3} \sqrt{\zeta / K} \bigr) + \perr, 
\$
where $\perr$ is defined in \eqref{eq:perr}.
\end{theorem}
% \begin{proof}
% See \S\ref{sec:main} for a detailed proof. 
% \end{proof}

Theorem~\ref{thm:main} states that the sequence of policies generated by pessimistic OPO converges sublinearly to an optimal policy in the regularized MDP with an additional pessimistic error term $\perr$. We remark that such an error term $\perr$ is irremovable, as it arises from the information-theoretic lower bound \cite{jinPessimismProvablyEfficient2021a}. 
Moreover, given the equivalence between offline OPO and Theo-SCORE in Lemma~\ref{lemma:dpg=ppo}, we know that Theo-SCORE also converges to an optimal policy under a sublinear rate. For the detailed proofs of the lemma and the theorem, please see Appendix~\ref{app:theory}.

% =======================================================
% 
%                    Related Work 
% 
% =======================================================
\section{Related Work}
\label{sec:related_work}
The techniques used in most existing works in offline RL algorithms fall into two categories, namely, policy-constrained methods and value-penalized methods. Policy-constrained methods defend against OOD actions by restricting the hypothesis space of the policy. For example, BCQ~\cite{fujimotoBenchmarkingBatchDeep2019a,fujimotoOffpolicyDeepReinforcement2019} and EMaQ~\cite{ghasemipourEmaqExpectedmaxQlearning2021} exclusively consider actions proposed by the estimated behavioral policy. Alternatively, some methods~\cite{wuBehaviorRegularizedOffline2019a, kumarStabilizingOffpolicyQlearning2019,nairAWACAcceleratingOnline2021,kostrikovOfflineReinforcementLearning2021b,wuUncertaintyWeightedActorCritic2021a} reformulate the policy optimization problem as a constrained optimization problem to keep the learned policy sufficiently close to the behavioral policy. More recently, Fujimoto et al.~\cite{fujimotoMinimalistApproachOffline2021a} proposes a simple yet effective solution by directly using behavioral cloning in policy optimization. In brief, policy-constraint approaches have an intuitive motivation but often struggle to outperform the behavioral policy. While SCORE's regularization term also adopts the policy-constraint technique, it is improved by a decaying factor that avoids an explicit reliance on the behavioral policy.

Alternatively, the value-penalized methods steer the policy towards training distribution by penalizing the value of OOD actions. For example, CQL~\cite{kumarConservativeQLearningOffline2020} explicitly minimizes the action value of OOD data via a value regularization term. MOPO~\cite{yuMOPOModelbasedOffline2020a} and MOReL~\cite{kidambiMorelModelbasedOffline2020} learn pessimistic dynamic models and use model uncertainty as the penalty. However, in a subsequent work~\cite{yuCOMBOConservativeOffline2021}, the authors claim that the model uncertainty for complex dynamics tends to be unreliable and revert to the CQL-type penalty method. Instead of explicitly learning the dynamics and then inferring its uncertainty, SCORE utilizes bootstrapped q-networks to estimate uncertainty, providing reliable estimations. 
% \vskip +0.1in

Some model-free methods also use uncertainty to construct the penalty term. For example, UWAC~\cite{wuUncertaintyWeightedActorCritic2021a} proposes to use Monte Carlo (MC) dropout to quantify uncertainty and perform uncertainty-weighted updates. This method relies on a strong policy-constrained method~\cite{kumarStabilizingOffpolicyQlearning2019}; more importantly, the dropout method does not converge with increasing data~\cite{osbandRandomizedPriorFunctions2018}. 
% Conversely, our approach reduces to a pure uncertainty-based method when the regularization coefficient $\lambda$ decays to zero, and with more data, the uncertainty decreases to zero. 
EDAC~\cite{anUncertaintyBasedOfflineReinforcement2021a} proposes to use a large number of networks to guarantee diversification so that OOD actions are sufficiently penalized via the traditional min-Q objective. PBRL~\cite{baiPessimisticBootstrappingUncertaintyDriven2022} introduces an OOD sampling procedure similar to CQL and a pseudo target punished by the uncertainty for OOD actions. 
Theoretically, PBRL analyzes the connection between bootstrapped uncertainty and the LCB penalty, but it does not explicitly address the algorithm's convergence rate. In contrast, SCORE is supported by a thorough theoretical analysis showing that it achieves a sublinear convergence rate. Since SCORE does not need to sample OOD data, its convergence result remains valid even without access to the exact value of OOD samples. 
% While these works attribute their success to defending against OOD actions, it lacks a rigorous mathematical connection to the suboptimality in offline RL.
% SCORE, by contrast, is directly motivated by reducing false correlations, which is derived by mathematically decomposing the suboptimality. It employs a simple regularization term to stabilize the learning process, providing high-quality uncertainty estimations with low computational costs (small number of networks and no need for OOD sampling). Such simplicity facilitates theoretical analysis and shortens the gap between theory and practice.

In addition to algorithmic advances, some research works study the detrimental effects of false correlation from the theoretical perspective. To deal with this problem, most existing works impose various assumptions on the sufficient coverage of the dataset, e.g., the ratio between the visitation measure of the target policy and that of the behavior policy to be upper bounded uniformly over the state-action space \cite{farajtabarMoreRobustDoubly2018a, liuBreakingCurseHorizon2018a, xieOptimalOffPolicyEvaluation2019, nachumDualDICEBehaviorAgnosticEstimation2019a, jiangMinimaxValueInterval2020a, duanMinimaxOptimalOffPolicyEvaluation2020, yangOffPolicyEvaluationRegularized2020a, yinNearOptimalProvableUniform2021a, zhangGenDICEGeneralizedOffline2022}, or the concentrability coefficient to be upper bounded \cite{scherrerApproximateModifiedPolicy2015a, chenInformationTheoreticConsiderationsBatch2019a, xieApproximationSchemesBatch2020a, fanTheoreticalAnalysisDeep2020a, xieBatchValuefunctionApproximation2021a, liaoBatchPolicyLearning2022}. 
Until recently, without assuming sufficient coverage of the dataset, Jin et al.~\cite{jinPessimismProvablyEfficient2021a} incorporate pessimism into value iteration to establish a data-dependent upper bound on the suboptimality in episodic MDPs. 
We take a step forward from the theoretical result in~\cite{jinPessimismProvablyEfficient2021a}, extending it to address infinite-horizon regularized MDPs and incorporating the policy optimization process, which bridges the gap between theory and practice. Our algorithm converges to an optimal policy with a sublinear rate, relying on a minimal assumption regarding the compliance of the dataset. Specifically, we assume the data collection process aligns with the MDP of interest.

% =======================================================
% 
%                    Experiments 
% 
% =======================================================
\section{Experiments}
\label{sec:experiments}
In this section, we conduct extensive experiments to verify the effectiveness of the propose algorithm. We first present the settings in Section~\ref{subsec:settings}, followed by the comparison results in Section~\ref{subsec:comparison}. Then we visualize and analyze the uncertainty learned by our method in Section~\ref{subsec:uncertainty}. Lastly, we discuss the results of ablation studies and hyperparameter analyzes in Section~\ref{subsec:ablation} and Section~\ref{subsec:hyperparameter} respectively.

\begin{figure}[t]
% \vskip -0.2in
\begin{center}
\includegraphics[width=0.43\textwidth]{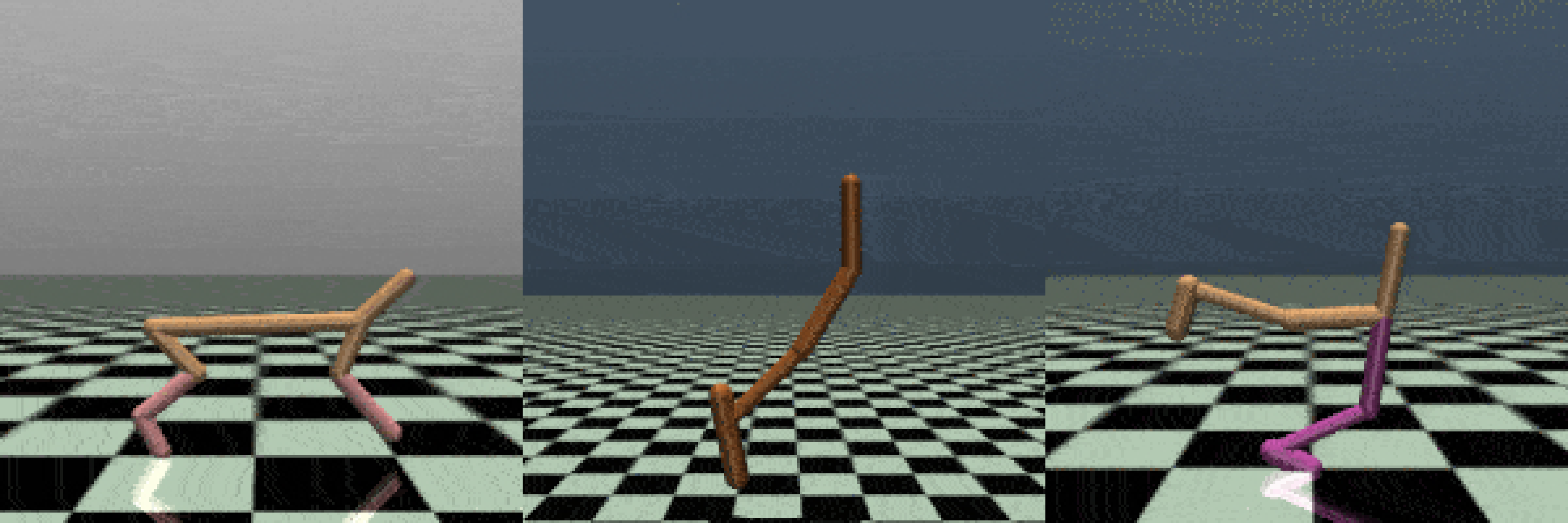}
\end{center}
\caption{An illustration of the robotic tasks: halfcheetah (left), hopper (middle), walker2d(right).}
\label{fig:d4rl_mujoco}
% \vskip -0.1in
\end{figure}
\subsection{Experimental Settings}
\label{subsec:settings}
We conduct extensive experiments on the widely adopted D4RL~\cite{fuD4RLDatasetsDeep2021} benchmark. 
D4RL provides three robotic tasks (halfcheetah, hopper, and walker2d) shown in \figurename~\ref{fig:d4rl_mujoco}. Each task has five datasets of different qualities:
``random'' is collected using a random initialized policy, ``medium'' is collected using a partially trained policy, and ``expert'' is collected using a well-trained policy. ``medium-replay'' and ``medium-expert'' are derived from a mixture of policies~\footnote{In practice, the behavioral policy is usually unknown and comes in the form of a mixture of multiple policies of different quality.}, with the former containing all the data in the replay buffer of the ``medium-level'' policy and the latter being the product of mixing the ``meidum'' and ``expert'' datasets in equal proportions. In brief, unlike online RL experiments that only test algorithms on different tasks, the offline RL benchmark evaluates algorithms on datasets of different qualities. In this way, we can comprehensively assess the ability of an algorithm to learn effective policies on various training data distributions. There are 15 datasets in total.

We evaluate 8 different baselines of diversified types on the D4RL benchmark, as summarized in \tablename~\ref{tab:characteristics}:
\begin{itemize}
    \item \textbf{BCQ}~\cite{fujimotoOffpolicyDeepReinforcement2019} is a simple yet strong baseline, providing stable performance. It involves training a Variational Auto-Encoder (VAE) and an agent with the actor-critic architecture. In particular, the actor perturbs the actions proposed by the VAE instead of directly output actions, and the critic determines the final action to be performed from a set of candidate actions.
    \item \textbf{BEAR}~\cite{kumarStabilizingOffpolicyQlearning2019} employs a similar architecture to BCQ but directly uses the actor to output actions. The authors propose to use the  Maximum Mean Discrepancy (MMD) distance to constrain the learning process in policy optimization to avoid large differences with the behavioral policy.
    \item \textbf{CQL}~\cite{kumarConservativeQLearningOffline2020} is the state-of-the-art offline RL algorithm. Unlike policy-constrained methods, it incorporates a strong value regularizer into the critic's loss to suppress the value of OOD actions, thus indirectly leading the agent to resist them. Specifically, the OOD data is sampled using the learned policy in practice, which requires additional computational costs.
    \item \textbf{MOPO}~\cite{yuMOPOModelbasedOffline2020a} is a model-based offline RL algorithm which modifies the observed reward by incorporating the maximum standard deviation of the learned models (model uncertainty) as a penalty. The agent is then trained with the penalized rewards.
    \item \textbf{MOReL}~\cite{kidambiMorelModelbasedOffline2020} is also a model-based offline RL algorithm. Instead of the using the maximum standard deviation, MOReL proposes to use the maximum disagreement of the learned models as the penalty.
    \item \textbf{UWAC}~\cite{wuUncertaintyWeightedActorCritic2021a} improves BEAR~\cite{kumarStabilizingOffpolicyQlearning2019} by incorporating MC dropout to estimate the uncertainty of the input sample and weight the loss accordingly.
    \item \textbf{TD3-BC}~\cite{fujimotoMinimalistApproachOffline2021a} is a simple offline reinforcement learning algorithm with minimal modifications to the TD3~\cite{fujimotoAddressingFunctionApproximation2018} algorithm. The authors suggest to use a weighted sum of the critic loss and the behavior cloning loss to update the actor. Despite its simplicity, the method also offers stable and good performance.
    \item \textbf{PBRL}~\cite{baiPessimisticBootstrappingUncertaintyDriven2022} quantifies uncertainty through the disagreements of the bootstrapped Q-networks and performs pessimistic updates. Similar to CQL, PBRL utilizes an OOD sampler to obtain pseudo-OOD samples. In the policy evaluation step, PBRL proposes to introduce an additional pseudo target for these samples and imposes a large uncertainty penalty on them. Inspired by~\cite{osbandRandomizedPriorFunctions2018}, PBRL-prior constructs a set of random prior networks of the same size as the Q-networks on top of PBRL. Empirically, the performance is remarkably improved by incorporating these prior networks. Therefore, we report the performance of the PBRL-prior version in the experiments.
\end{itemize}

We remark that the results reported in the D4RL white paper~\cite{fuD4RLDatasetsDeep2021} are based on the ``v0" version of D4RL, and most previous work reuse those results for comparison. However, the ``v0" version has some errors that may lead to wrong conclusions and the authors release a ``v2" version later to correct the errors. To this end, we rerun all the baseline methods on the ``v2'' version and report the new results.

\begin{table}[t]
\centering
\caption{\footnotesize Characteristics of recent offline RL algorithms. The notations `P', `V', `U', 'O', and `M' represent `Policy-constrained', `Value-penalized', `Uncertainty-aware', 'OOD sampling' and `Model-based' respectively.}
\label{tab:characteristics}
\resizebox{0.3\textwidth}{!}{
\begin{tabular}{llllll}
\toprule
{} & P & V & U & O & M\\ 
\midrule
BCQ~\cite{fujimotoOffpolicyDeepReinforcement2019}  & $\surd$ & $\times$ & $\times$ & $\times$ & $\times$\\
BEAR~\cite{kumarStabilizingOffpolicyQlearning2019} & $\surd$ & $\times$ & $\times$ & $\times$ & $\times$\\
MOPO~\cite{yuMOPOModelbasedOffline2020a}            & $\times$ & $\surd$ & $\surd$ & $\times$ & $\surd$\\
MOReL~\cite{kidambiMorelModelbasedOffline2020}     & $\times$ & $\surd$ & $\surd$ & $\times$ & $\surd$\\
CQL~\cite{kumarConservativeQLearningOffline2020}    & $\times$ & $\surd$ & $\times$ & $\surd$ & $\times$\\
UWAC~\cite{wuUncertaintyWeightedActorCritic2021a}  & $\surd$ & $\times$ & $\surd$ & $\times$ & $\times$\\
TD3-BC~\cite{fujimotoMinimalistApproachOffline2021a} & $\surd$ & $\times$ & $\times$ & $\times$ & $\times$\\
PBRL~\cite{baiPessimisticBootstrappingUncertaintyDriven2022} & $\times$ & $\surd$ & $\surd$ & $\surd$ & $\times$\\
SCORE (Ours) & $\surd$ & $\surd$ & $\surd$ & $\times$ & $\times$\\
\bottomrule
\end{tabular}}
% \vskip -0.2in
\end{table}
All the experiments are run on a single NVIDIA Quadro RTX 6000 GPU. To guarantee a fair comparison, we conducted the experiments under the same experimental protocol. The whole training process has a total of 1 million gradient steps, which are divided into 1000 epochs. At the end of each epoch, we run 10 episodes for evaluation. To reduce the effect of randomness on the experimental results, all experiments are run with 5 independent random seeds while keeping other factors unchanged.

\begin{figure*}[!t]
\begin{center}
\includegraphics[width=0.95\linewidth]{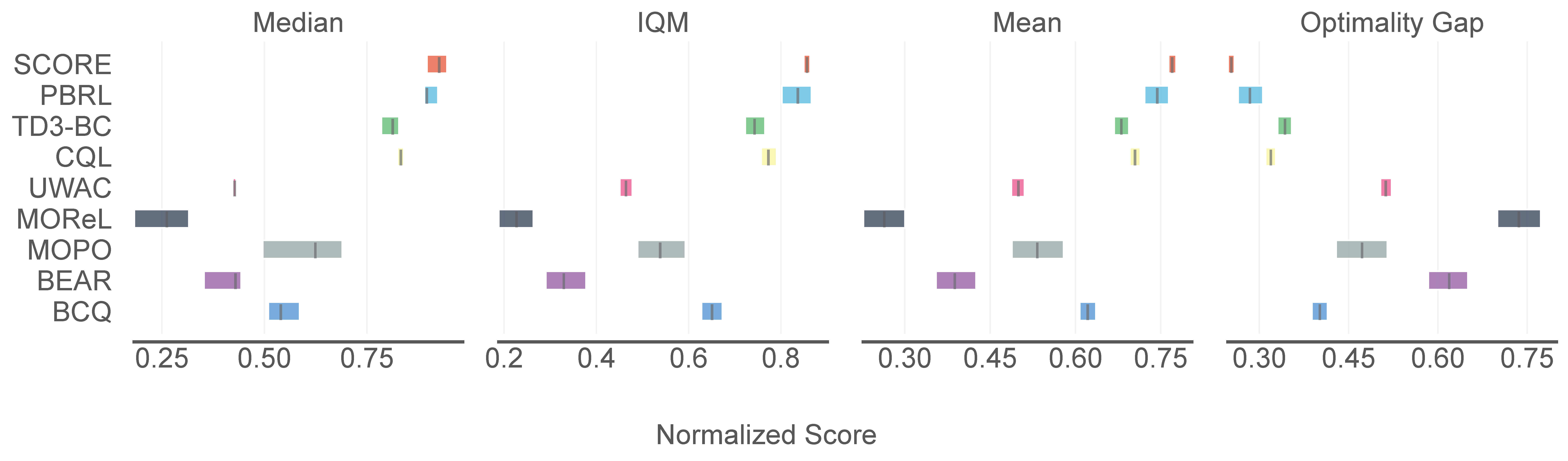}
\end{center}
\vskip -0.2in
\caption{\textbf{Aggregate metrics on D4RL-MuJoCo} with 95\% CIs based on 15 tasks. Higher mean, median, and interquartile mean (IQM) scores and lower optimality gap are better. The CIs are estimated using the percentile bootstrap with stratified sampling. IQM typically results in smaller CIs than median scores. All results are based on 5 runs (with different random seeds) per task.} % Large values of mean scores relative to median and IQM indicate being dominated by a few high performing tasks (e.g. MOPO). Optimality gap is less susceptible to outliers compared to mean scores.
\label{fig:d4rl_mujoco_aggregates}
\vskip -0.2in
\end{figure*}
% \begin{wrapfigure}{r}{0.43\textwidth}
\begin{figure}[t]
% \vskip -0.2in
\begin{center}
\includegraphics[width=0.43\textwidth]{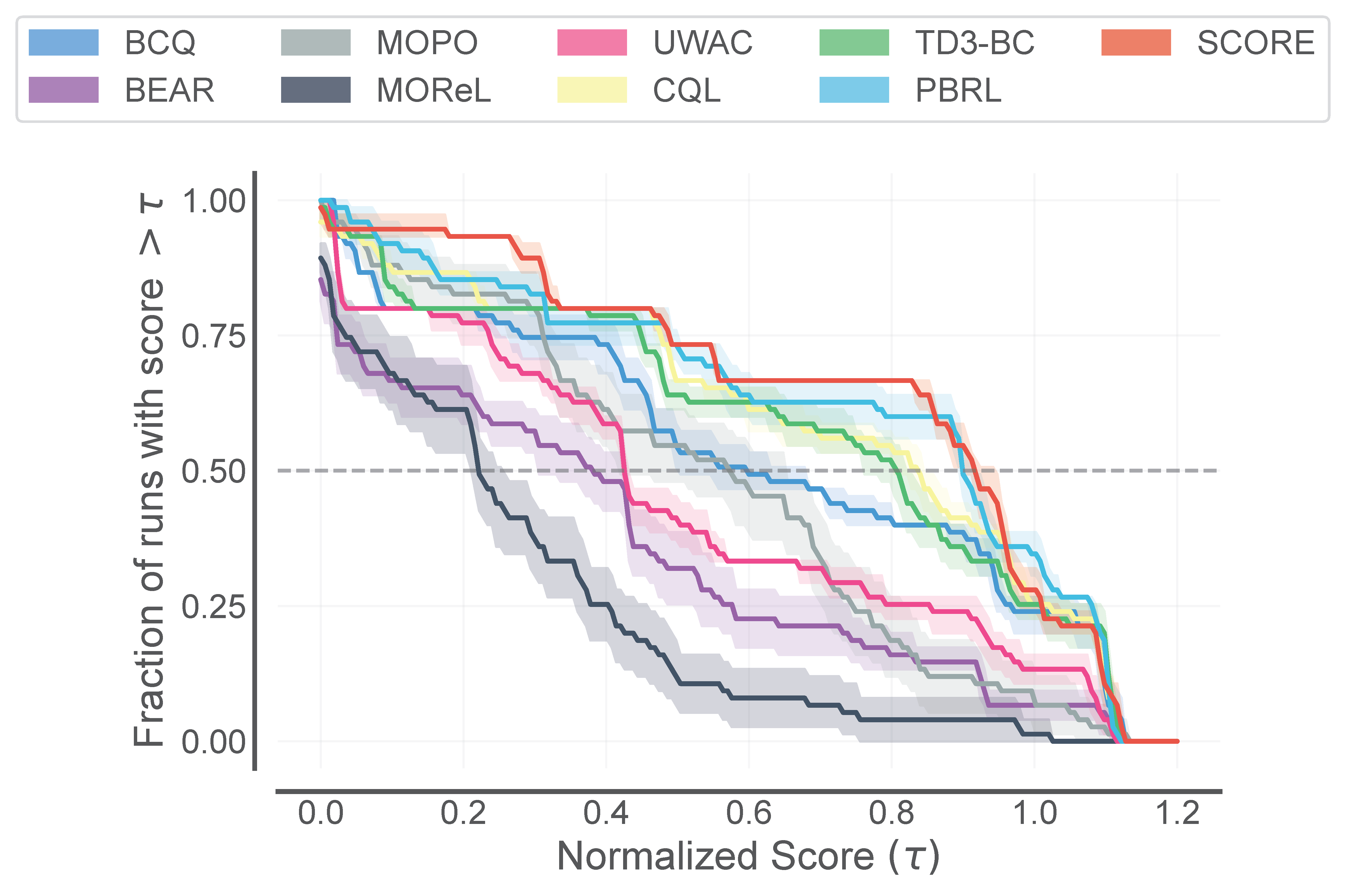}
\end{center}
\vskip -0.2in
\caption{\footnotesize \textbf{Performance profiles on D4RL-MuJoCo}. The distributions are computed using normalized scores obtained after training for 1M steps. Shaded regions show 95\% percentile stratified bootstrap CIs. The larger the area under the curve, the better.}
\label{fig:d4rl_mujoco_performance_profiles}
% \vskip -0.1in
\end{figure}
% \end{wrapfigure}
We specify both the actor network and the critic network to be a three-layer neural network with 256 neurons per layer. The first two layers of both networks use the ReLU activation function. In particular, the last layer of the actor network uses the tanh activation function for outputting actions. Both networks have a separated Adam optimizer and the learning rate is 3e-4. Some baselines differ in settings and have additional hyperparameters. We use the official code and the suggested settings. When the default parameters do not perform well, we use grid search to find the optimal hyperparameters for different datasets.

\subsection{Comparison Experiments}
\label{subsec:comparison}
As suggested by the NeurIPS 2021 outstanding paper~\cite{agarwalDeepReinforcementLearning2021}, we analyze the experimental results via the reliable evaluation suite~\footnote{\url{https://github.com/google-research/rliable}}, which reports the interval estimations of the overall performance instead of the point estimations. As a result, it can effectively prevent the researcher from making unreliable conclusions. We present all four aggregate metrics with the corresponding confidence intervals. Specifically, the IQM score is regarded as a more robust aggregate metric than the commonly used mean score because it is robust to outlier scores and is more statistically efficient. 
The detailed numerical results are available in the Appendix~\ref{app:tables}. According to \figurename~\ref{fig:d4rl_mujoco_aggregates} and \figurename~\ref{fig:d4rl_mujoco_performance_profiles}, the two best-performing baselines are CQL and PBRL, while SCORE performs the best consistently in all four metrics. Unlike CQL, SCORE and PBRL differentiate the action-value using uncertainty, allowing refined value penalty and better performance. While PBRL is motivated by defending against OOD actions and reducing extrapolation errors. The uncertainty penalty does more than that, as shown in Section
~\ref{sec:preliminaries}. PBRL heavily relies on the OOD data sampling procedure during the learning process, which is quite time-consuming. On the other hand, SCORE identifies the core issue as the inability to obtain high-quality uncertainty estimations. We propose using a simple regularizer to effectively address this issue, enabling a reliable implementation of the pessimism principle. According to \tablename~\ref{tab:runtime}, SCORE has a much lower computational cost than CQL and PBRL due to its simplicity. By removing the OOD samplers from CQL and PBRL, SCORE significantly accelerates the training process without sacrificing the performance.

% \begin{wraptable}{r}{0.4\textwidth}
\begin{table}[t]
\vskip -0.1in
\centering
\caption{\footnotesize Comparison of the computational costs of the three best-performing methods.}
\label{tab:runtime}
\resizebox{0.38\textwidth}{!}{
\begin{tabular}{llll}
\toprule
{} & CQL & PBRL & SCORE\\ 
\midrule
Num of parameters & 0.69M & 5.66M & 0.87M\\
Run time per epoch & 32$\sim$35s & 52$\sim$55s & 16$\sim$18s\\
\bottomrule
\end{tabular}}
\end{table}
% \vskip -0.2in
% \end{wraptable}
Empirically, SCORE is particularly effective on low- and medium-quality datasets. While the datasets do not satisfy the sufficient coverage assumption since the behavioral policies in these settings are highly suboptimal, they still include some positive behaviors, so it is possible to learn better policies than the behavioral ones. The key to mastering these good behaviors is to avoid suboptimal actions that lead to high returns by chance. While many existing methods focus on the extrapolation errors caused by OOD actions, SCORE strives to address false correlations, a border issue, offering further improvements.
% We also conducted experiments on the more challenging D4RL-Adroit task suite, and SCORE still performs well compared with other methods. We refer to Appendix~\ref{app:exp} for more details.

\begin{figure}[t]
\centering
\subfloat[halfcheetah]{
\label{subfig:uncertainty_halfcheetah}
\includegraphics[width=0.15\textwidth]{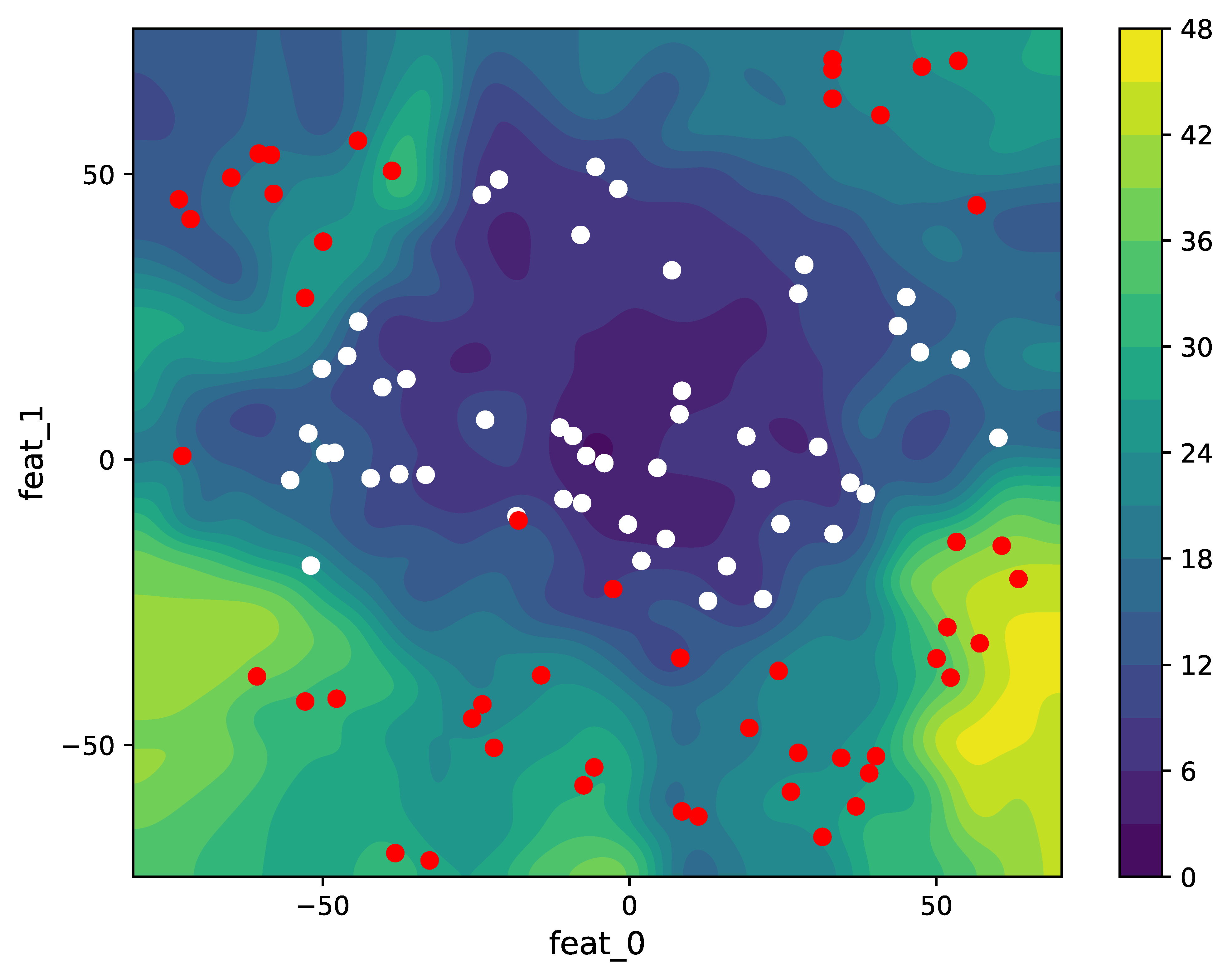}}
%\quad
\subfloat[hopper]{
\label{subfig:uncertainty_hopper}
\includegraphics[width=0.15\textwidth]{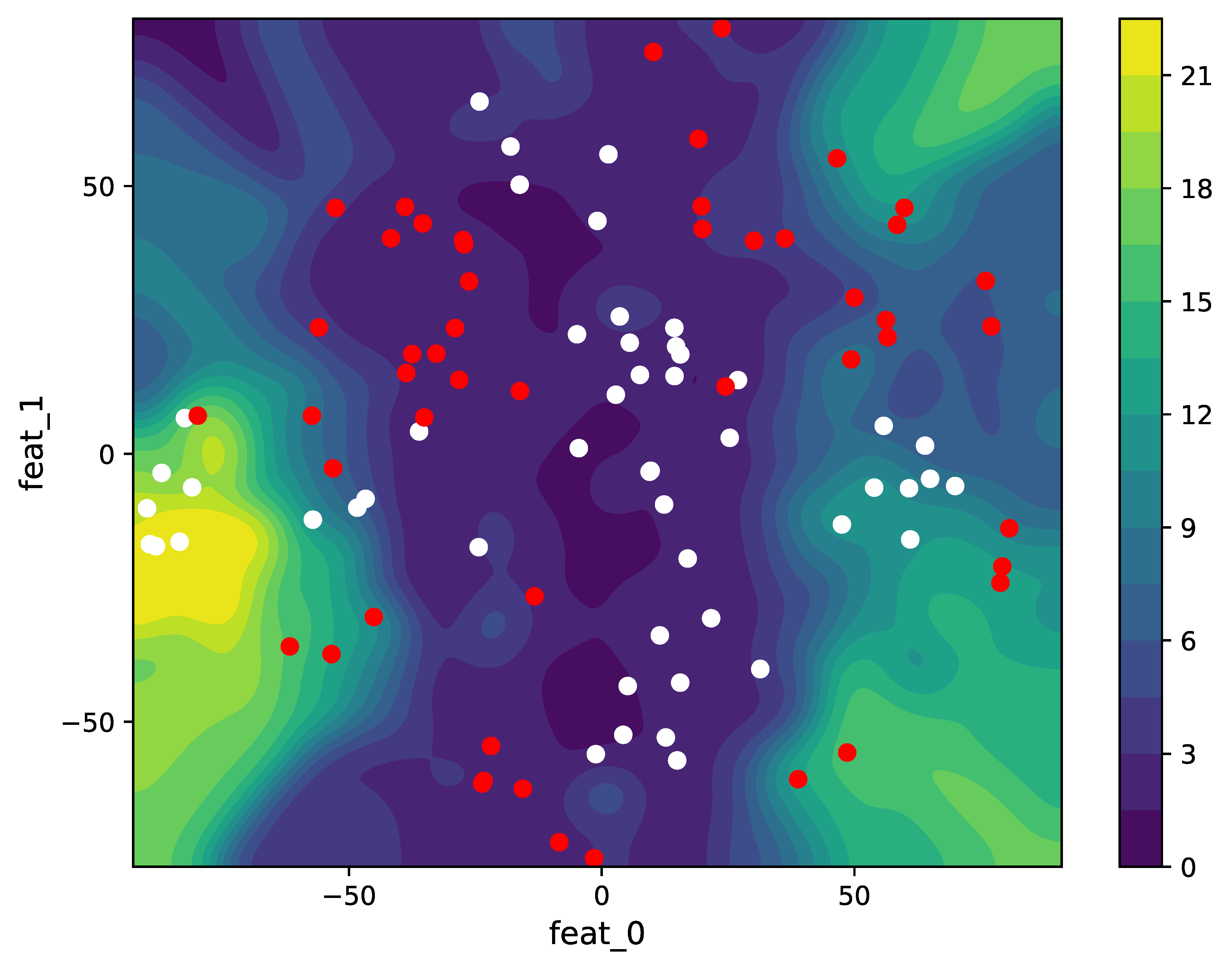}}
%\quad
\subfloat[walker2d]{
\label{subfig:uncertainty_walker2d}
\includegraphics[width=0.15\textwidth]{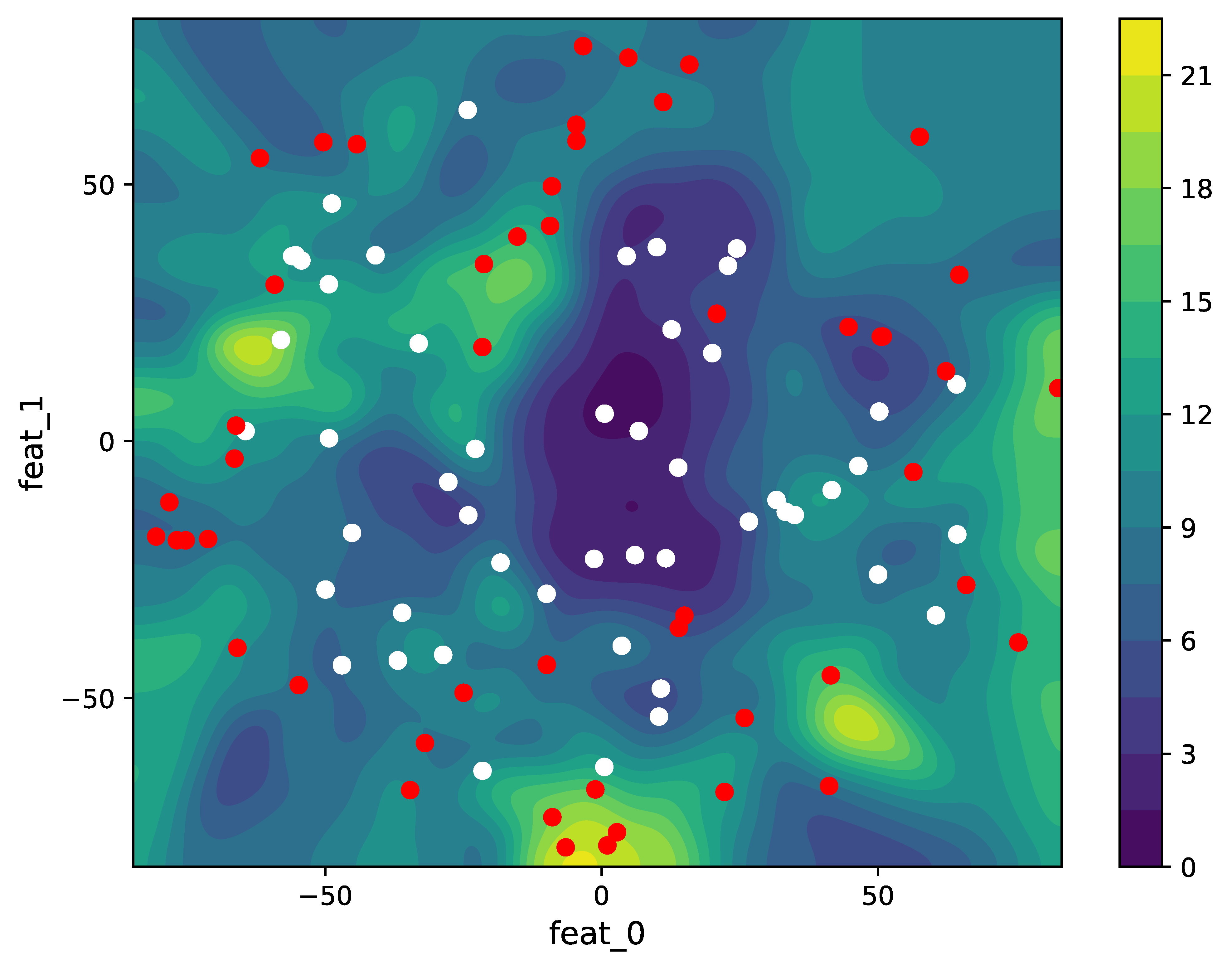}}
\vskip -0.12in
\caption{\footnotesize Uncertainty estimation of in-distribution (white dots) and OOD (red dots) samples.}
\label{fig:uncertainty}
\end{figure}
\subsection{Visualization and Analysis of Uncertainty}
\label{subsec:uncertainty}
To gain a better insight into SCORE, we visualize the learned uncertainty. Specifically, we apply the Q functions trained on the medium-replay dataset to quantify uncertainty. The in-distribution samples are drawn from the medium-replay dataset, and the OOD samples are from the expert dataset. For visualization purposes, we reduce the features of these samples to two dimensions using t-distributed stochastic neighbor embedding (t-SNE). \figurename~\ref{fig:uncertainty} shows the contour plot of the uncertainty on the two-dimensional feature space, in which the white dots denote in-distribution samples and the red dots correspond to OOD samples.

\begin{figure}[t]
\begin{center}
\includegraphics[width=0.43\textwidth]{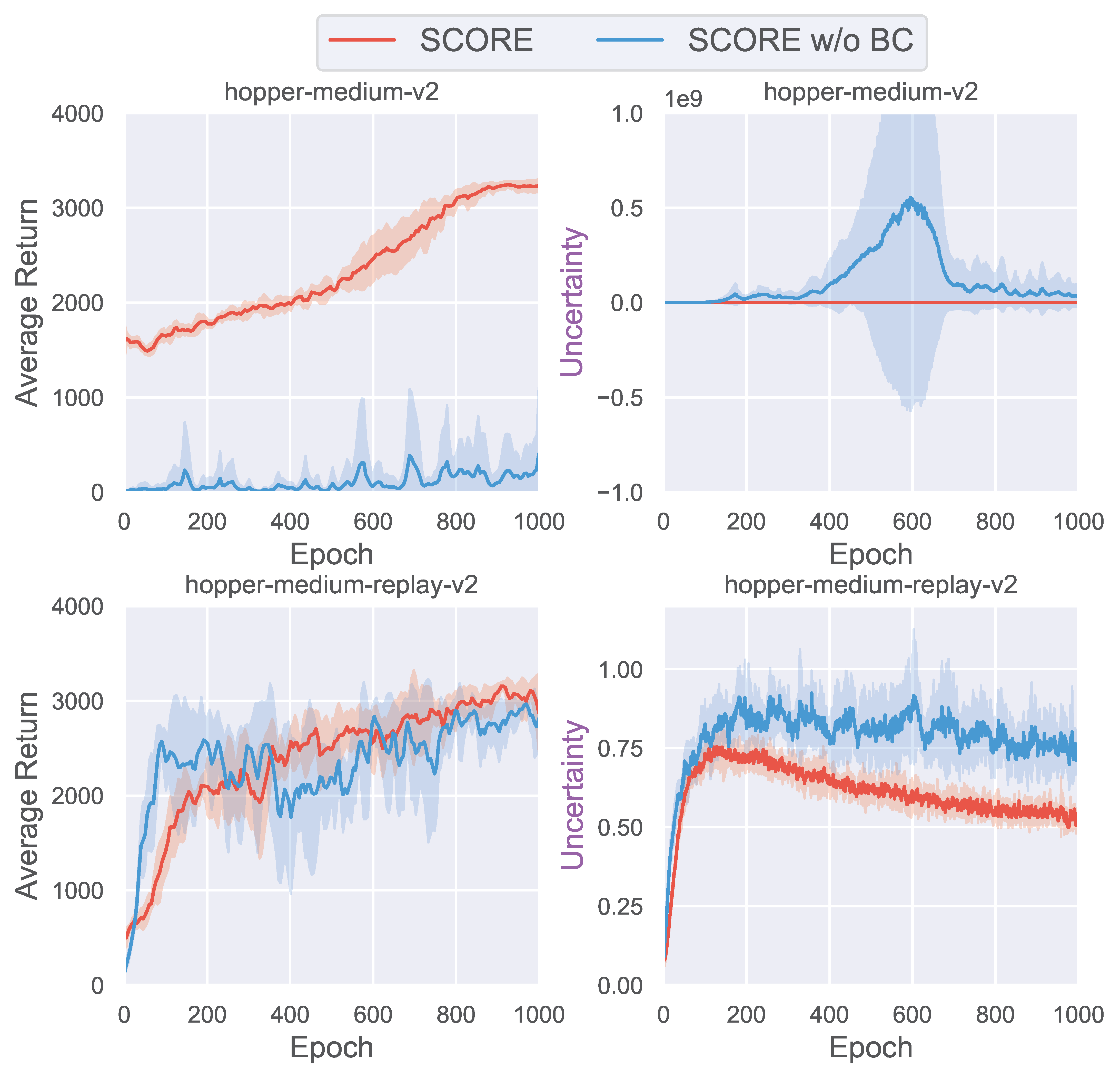}
\end{center}
\vskip -0.2in
\caption{The average return and the uncertainty of SCORE with or without the proposed regularizer.}
\label{fig:ablation_bc}
\end{figure}
At first glance, the in-distribution samples (white) are more concentrated in regions with low uncertainty (the dark regions), while the OOD samples (red) loosely distribute in regions with higher uncertainty (the bright regions). We remark that the in-distribution and OOD samples are more easily distinguishable on halfcheetah, while the opposite holds for hopper and walker2d. This phenomenon surprisingly matches the performance observed in the experiments. 
On halfcheetah, the performance on the medium-replay dataset is substantially lower than on the expert dataset, while it is much closer on hopper and walker2d. We suggest that this phenomenon reflects the property of the dataset, i.e., the medium-replay datasets of hopper and walker2d may have better coverage of the state-action pairs induced by the expert policy. Thus, algorithms are more likely to learn high-level policies from these two medium-quality datasets.

\subsection{Ablation Studies}
\label{subsec:ablation}
To better understand the effects of the proposed regularizer and the uncertainty penalty, we conduct experiments that remove or replace them with other alternatives in this section.

\begin{figure}[t]
\begin{center}
\includegraphics[width=0.43\textwidth]{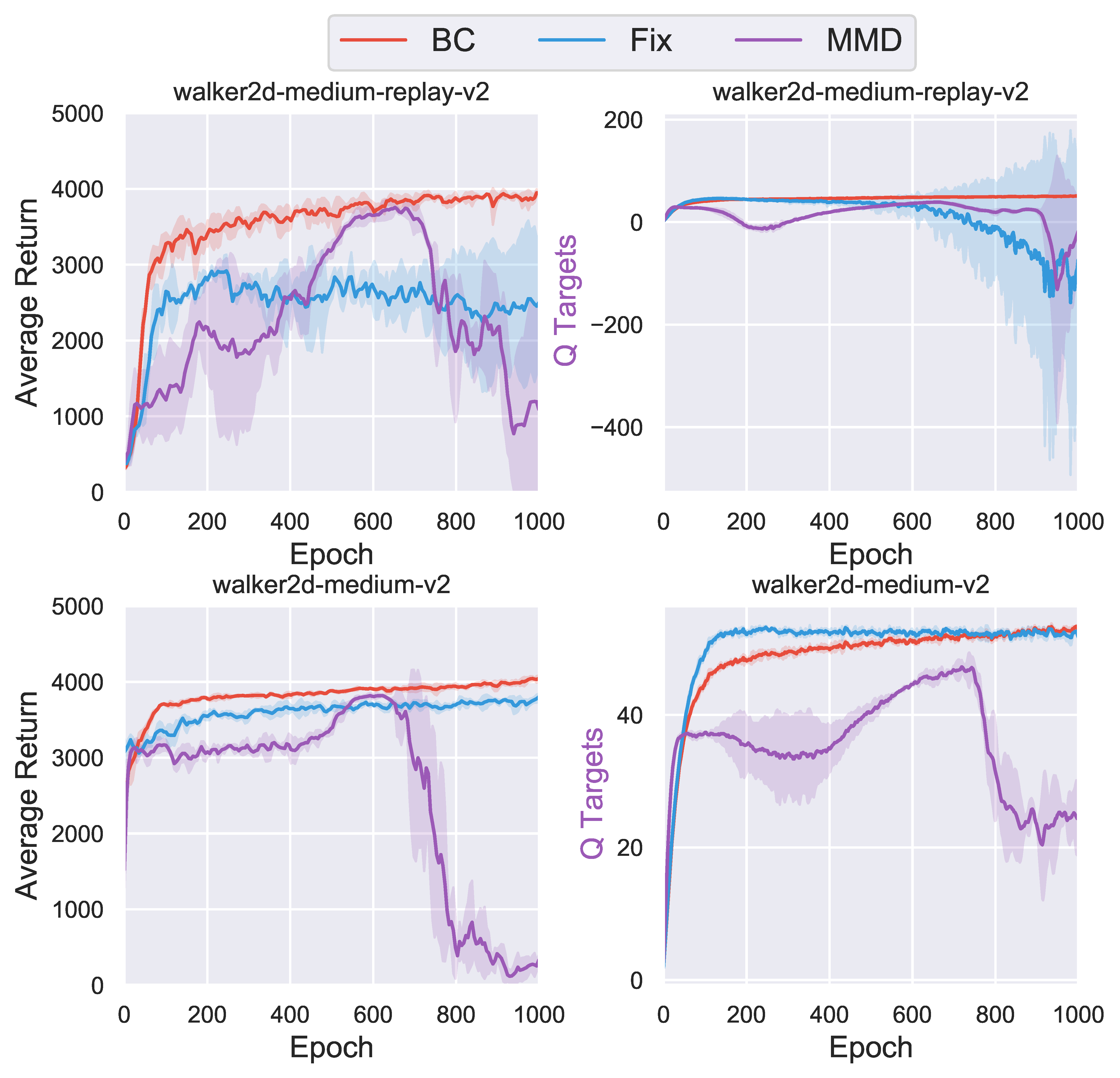}
\end{center}
\vskip -0.2in
\caption{The average return and the Q target of SCORE using different types of policy constraints. "BC" stands for our method, "Fix" uses a VAE to fit the behavior policy and then use it to propose actions (similar to BCQ~\cite{fujimotoOffpolicyDeepReinforcement2019}), and "MMD" constrains the learned policy to stay close to the estimated behavior policy under MMD (similar to BEAR~\cite{kumarStabilizingOffpolicyQlearning2019}).}
\label{fig:ablation_pc}
\end{figure}
\noindent \textbf{Regularization}. As pointed out in previous studies~\cite{fujimotoOffpolicyDeepReinforcement2019,levineOfflineReinforcementLearning2020a}, uncertainty-based methods empirically fail in offline RL. Moreover, estimating calibrated uncertainty for neural networks is a challenging task. \figurename~\ref{fig:ablation_bc} shows the ablation of the proposed BC regularizer. We first note that it plays a crucial role on narrow dataset distributions, e.g., the hopper-medium dataset collected by a single policy. Without this regularizer, the estimated uncertainty climbs rapidly at the beginning of the training process and then keeps fluctuating dramatically between enormous values. Meanwhile, we observe a remarkably different pattern when using the regularizer, with uncertainty tending to stabilize and slowly decrease after an initial rising phase. Even when the data are sufficiently diverse (less likely to suffer from OOD queries), e.g., the hopper-medium-replay dataset, regularization still helps stabilize uncertainty. These results reveal that the proposed regularizer contributes to reliable uncertainty estimations. Further, such high-quality uncertainty estimations drive SCORE to reduce false correlations and achieve superior performance.

Next, we investigate the effects of different policy constraint methods. In addition to using behavioral cloning, we study 1) using a VAE to fit the behavioral policy and propose actions, and 2) using a VAE to fit the behavioral policy and then constrain the learned policy to stay close to it during the policy optimization process. \figurename~\ref{fig:ablation_pc} reports the results. We can see that fitting a behavioral policy is usually a bad idea when the dataset is collected by a mixture of policies, e.g., the medium-replay datasets, since these policies sometimes conflict. As a result, proposing actions or constructing constraints based on the fitted behavioral policy may lead to suboptimal behavior. Even on datasets collected by a single policy, e.g., the medium datasets, using the fitted behavior policy is undesirable as it causes the agent to prematurely converge on suboptimal behavior. Overall, we find the proposed BC regularizer is simple yet effective.

\begin{figure}[t]
\begin{center}
\includegraphics[width=0.43\textwidth]{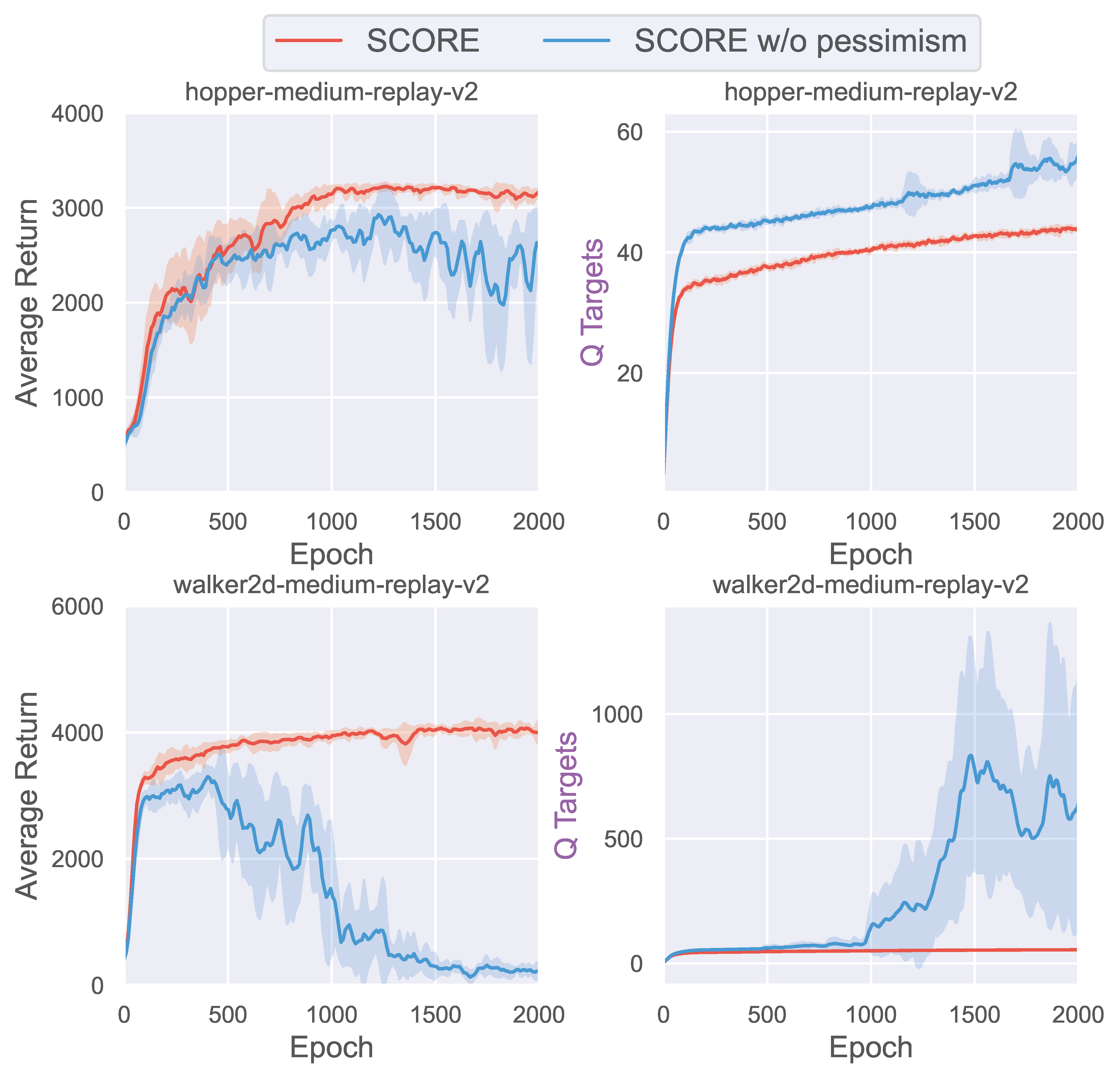}
\end{center}
\vskip -0.2in
\caption{The average return and the Q target of SCORE with and without the uncertainty penalty.}
\label{fig:ablation_pe}
\end{figure}
\begin{figure}[t]
\begin{center}
\includegraphics[width=0.43\textwidth]{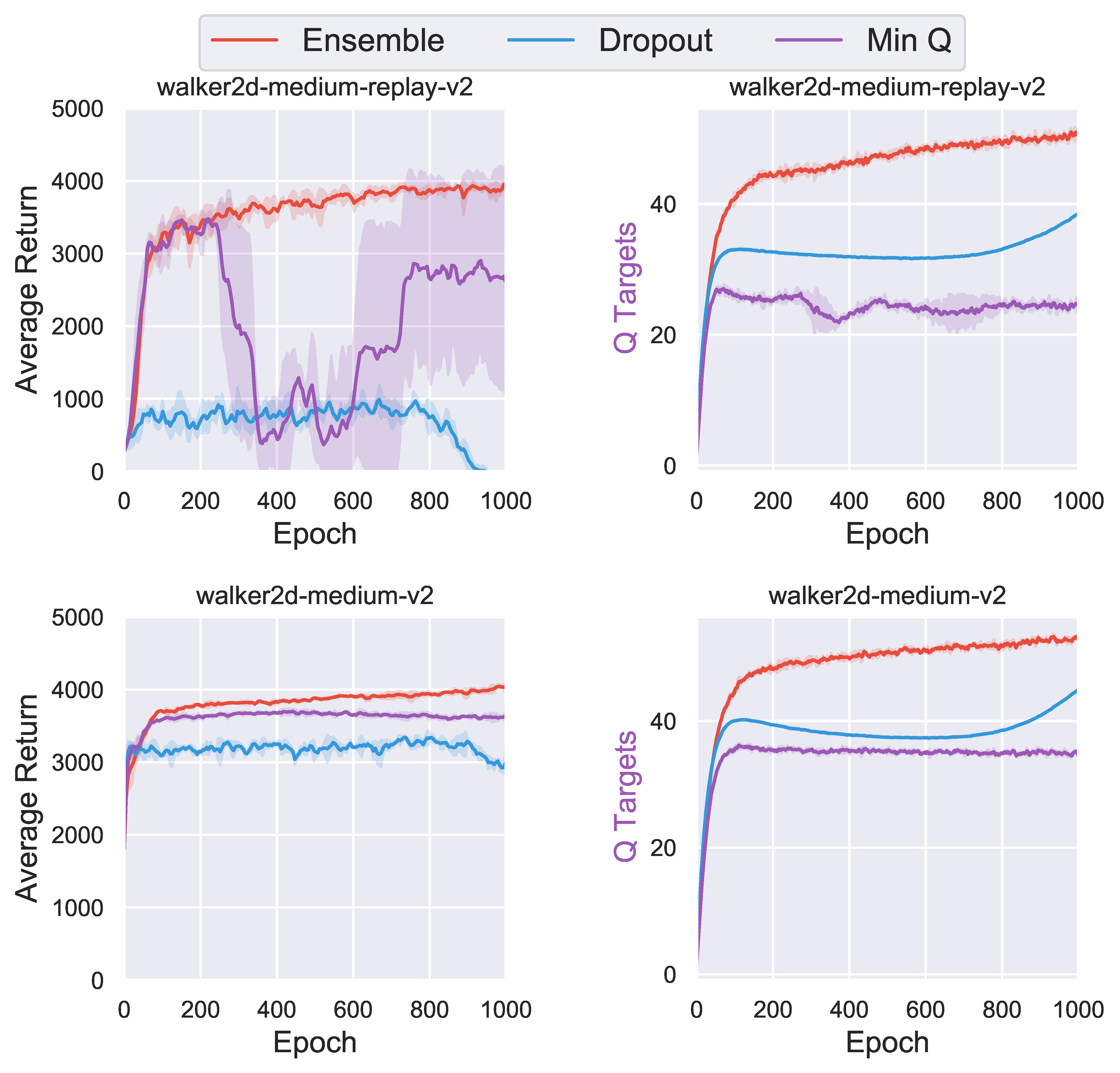}
\end{center}
\vskip -0.2in
\caption{The average return and the Q target of SCORE using different types of uncertainty qualification techniques. "Ensemble" stands for our method, "Dropout" means MC dropout, and "Min Q" ensures pessimism by using the smallest Q estimation.}
\label{fig:ablation_uq}
\end{figure}
\noindent \textbf{Pessimism}. In theory, false correlation is a critical factor that induces suboptimality in offline reinforcement learning, and pessimism in the face of uncertainty is a provably efficient solution to eliminate false correlations. \figurename~\ref{fig:ablation_pe} shows the difference between SCORE with and without pessimism. On the hopper-medium-replay dataset, removing the uncertainty penalty affects the stability of the training process and the convergence level. On the walker2d-medium-replay dataset, this even causes severe degradation. Through the curves of Q-values, we find that the performance drop when there is a jitter or explosion in Q-values. The agent without pessimism underperforms on both datasets and shows high Q-values. This is consistent with the intuition we established in Section~\ref{subsec:example}, i.e., an agent without knowledge of uncertainty is easily misled by false correlations, which makes the agent believe suboptimal actions have higher value and therefore acquire suboptimal policies. Overall, these results suggest that the regularizer alone cannot effectively solve the offline RL problem and that the pessimism principle is indispensable.

We also test other methods for quantifying uncertainty. SCORE quantifies uncertainty based on the bootstrapped ensemble method and uses it to construct the penalty term. Alternatively, one can use the MC dropout method to quantify uncertainty. In addition, one of the simplest ways to achieve pessimism is to use the minimum Q-value of multiple critics as the target. \figurename~\ref{fig:ablation_uq} illustrates the experimental results. We can see that both alternatives are inferior to the ensemble method, and they perform poorly on the medium-replay dataset generated by a mixture of multiple policies. As discussed in~\cite{osbandRandomizedPriorFunctions2018}, the dropout distribution does not concentrate with more observed data. Therefore, MC dropout fails to provide the epistemic uncertainty needed to eliminate false correlations in offline RL. As for the Min Q method, it generally needs more networks (accompanied by higher computation cost) and additional techniques to guarantee the diversity of the critics~\cite{anUncertaintyBasedOfflineReinforcement2021a}; otherwise, the multiple critics trained using the same target values may degenerate and fail to perceive uncertainty. Overall, the bootstrapped ensemble method performs well in measuring epistemic uncertainty and ensures SCORE's strong performance.

\begin{figure}[t]
\begin{center}
\includegraphics[width=0.48\textwidth]{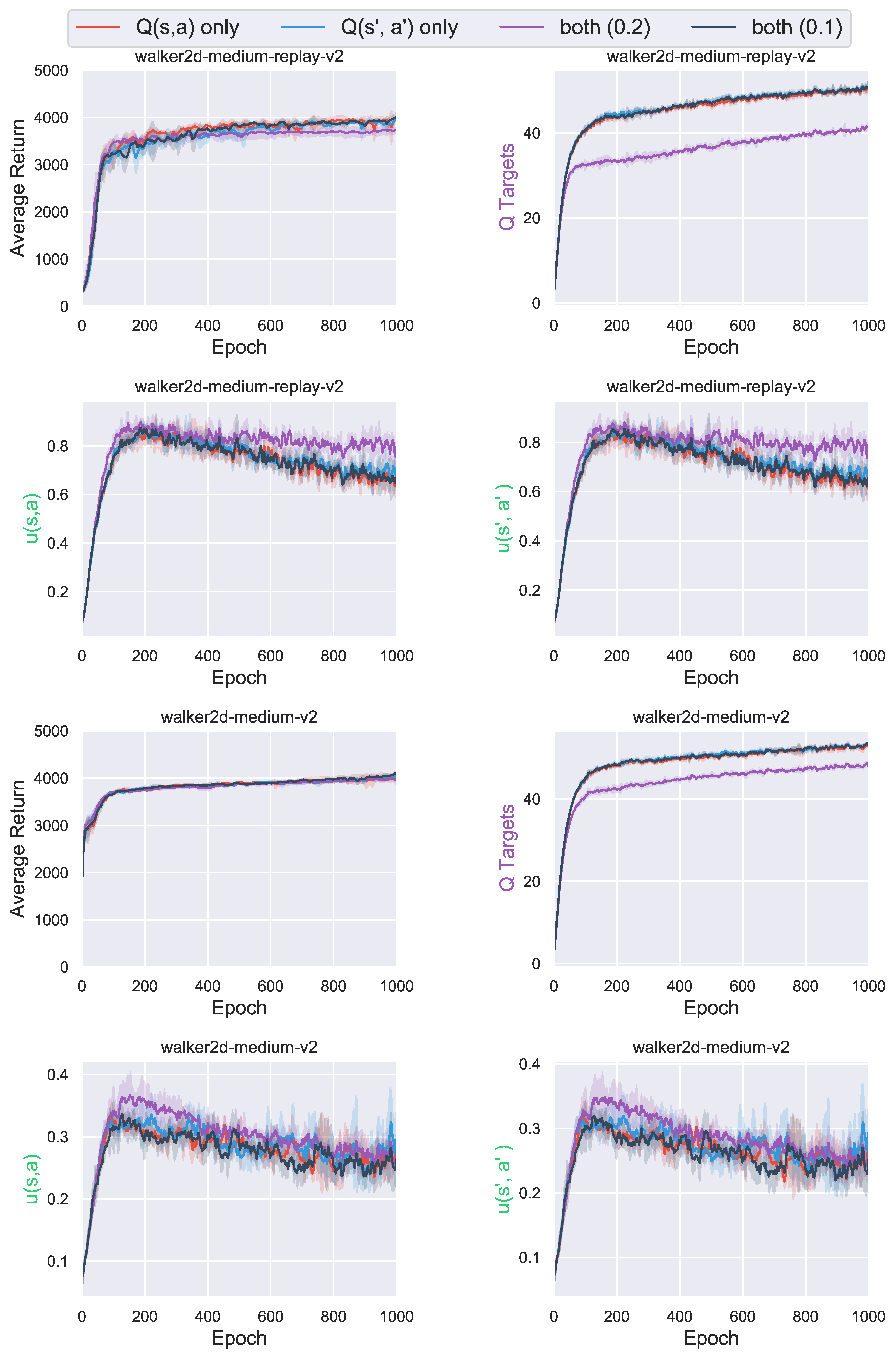}
\end{center}
\vskip -0.2in
\caption{The average return and the Q target of SCORE with different types of penalties. "$Q(s,a)$ only" and "$Q(s',a')$ only" mean the uncertainty penalty is computed only using $Q(s,a)$ or $Q(s',a')$, respectively ($\beta=0.2$). "both ($\beta$)" means the two penalty terms are used simultaneously. "both (0.1)" controls the overall strength penalty to be the same as "$Q(s,a)$ only" and "$Q(s',a')$ only", while "both (0.2)" doubles the strength. }
\label{fig:hyperparameter_penalty}
\end{figure}
\noindent \textbf{Penalty}. Remark that Equation~\ref{eq:q_objective} uses $u(s,a)$ for penalization. We note that researchers tried different types of penalties in previous work~\cite{baiPessimisticBootstrappingUncertaintyDriven2022}, including $u(s',a')$ or the combination of $u(s,a)$ and $u(s',a')$. Empirically, PBRL~\cite{baiPessimisticBootstrappingUncertaintyDriven2022} works best with $u(s',a')$. 
We conduct a similar experiment on SCORE to further study the penalty term. \figurename~\ref{fig:hyperparameter_penalty} demonstrates that SCORE is robust to different types of penalties, showing similar performance. 
This result motivates us to further analyze the difference between $u(s,a)$ and $u(s',a')$. $u(s,a)$ is computed based on $Q(s,a)$, which can further decompose as $r(s,a)+\mathbb{E}_{s'} \max_{a'} Q(s', a')$. 
We note that the Mujoco robotic tasks use a deterministic environment, meaning there is no missing information in $r(s,a)$ and $s'$. 
Therefore, if the agent is well-trained, the uncertainty $u(s,a)$ should depend only on $Q(s',a')$, which relates it to $u(s',a')$. This explains why ``$Q(s,a)$ only", ``$Q(s',a')$ only", and ``both (0.1)" show similar performance, Q values, and even uncertainty. 
We suggest that $u(s',a')$ works the best in~\cite{baiPessimisticBootstrappingUncertaintyDriven2022} because it directly fixes the uncertainty of $r(s,a)$ and $s'$ to be zero, avoiding the noise that may affect the training process. Overall, these experimental results support that SCORE obtains reliable uncertainty. 
The estimated uncertainty captures the long-term effects of the decisions; otherwise, it would equal zero in a deterministic environment.

\begin{figure}[t]
\begin{center}
\includegraphics[width=0.43\textwidth]{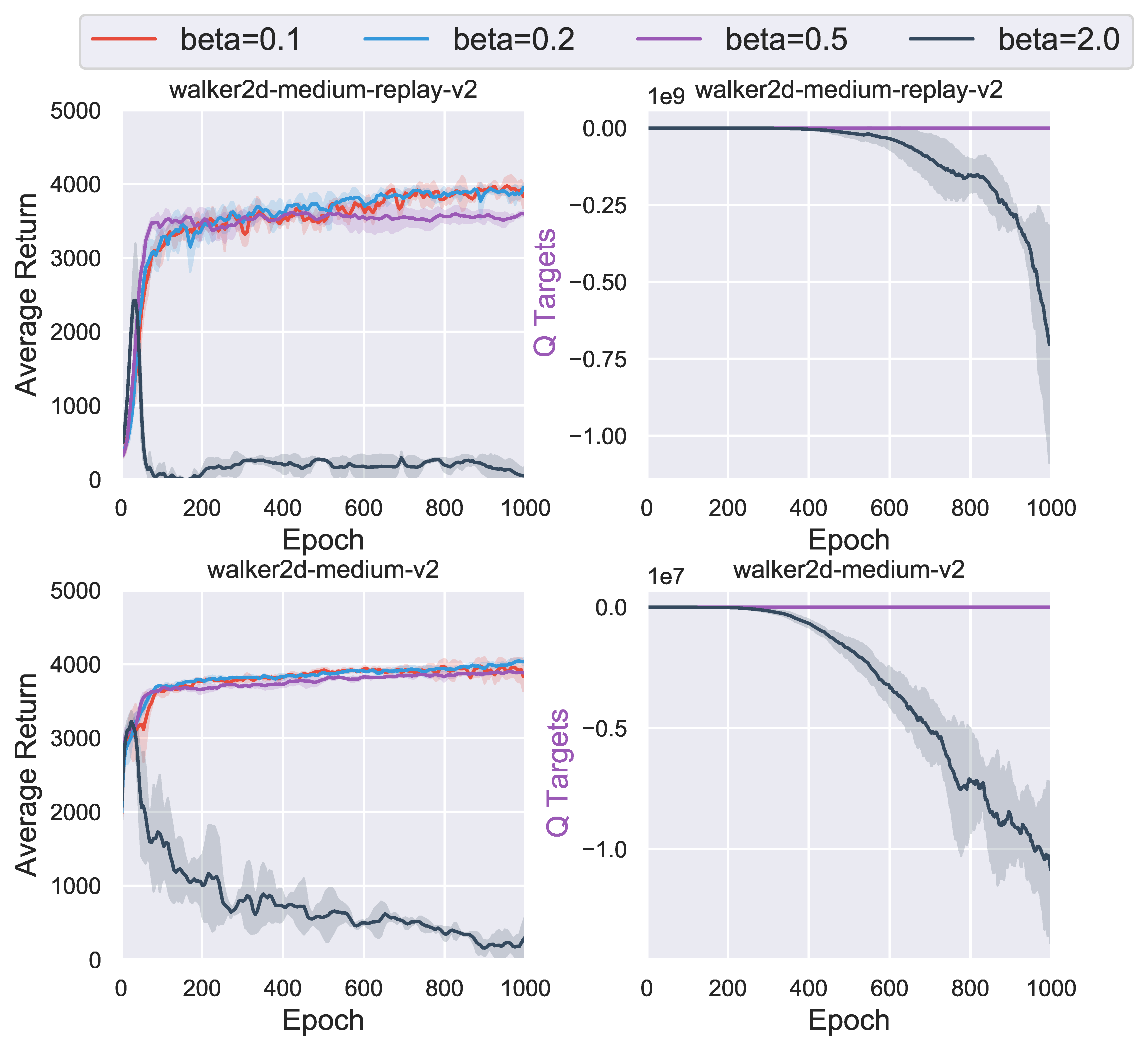}
\end{center}
\vskip -0.2in
\caption{The average return and the Q target of SCORE with different penalty coefficient $\beta$.}
\label{fig:hyperparameter_beta}
\end{figure}
\begin{figure}[t]
\begin{center}
\includegraphics[width=0.43\textwidth]{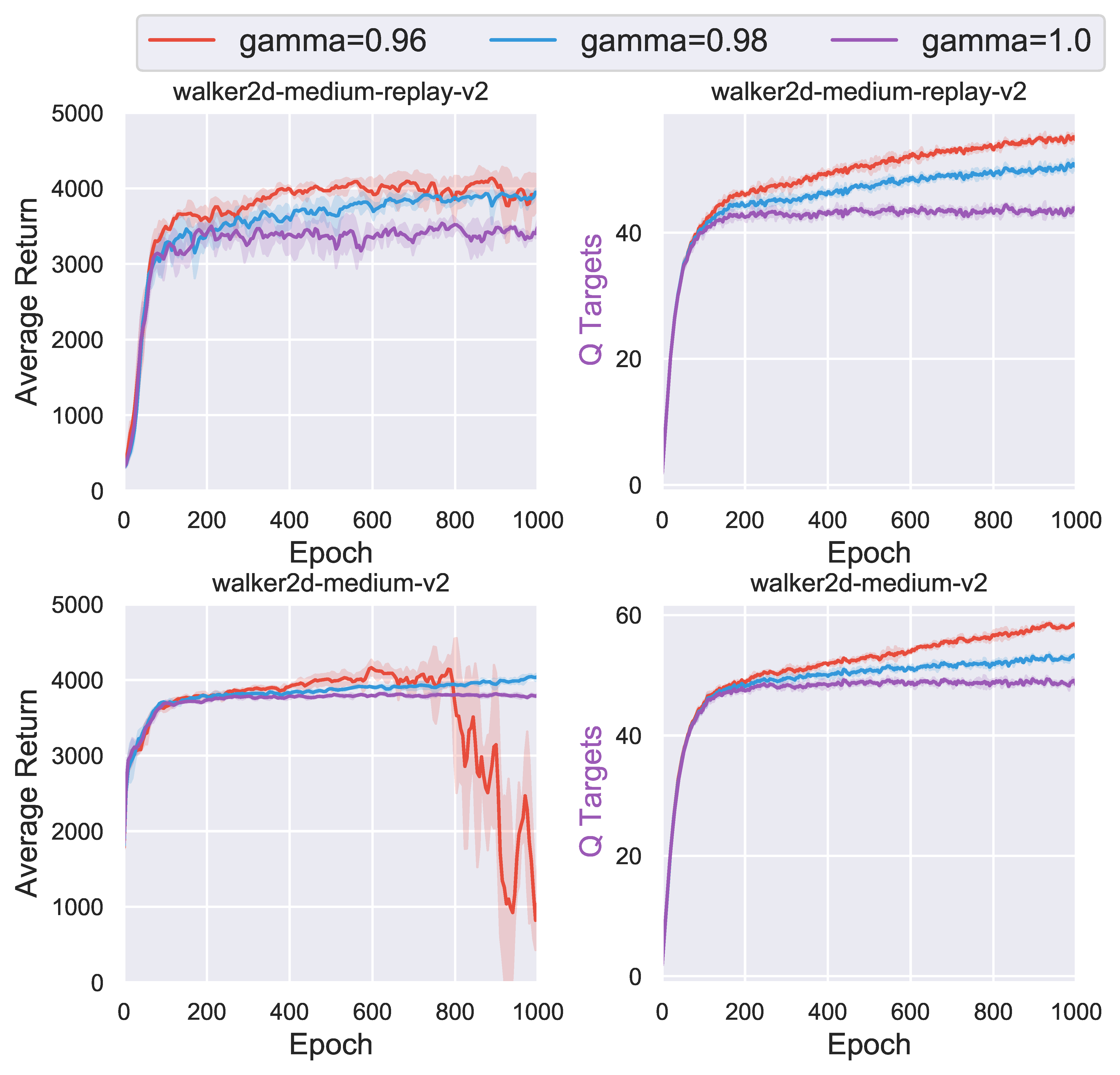}
\end{center}
\vskip -0.2in
\caption{The average return and the Q target of SCORE with different decaying factor $\gamma_{\text{bc}}$.}
\label{fig:hyperparameter_gamma}
\end{figure}
\begin{figure}[t]
\begin{center}
\includegraphics[width=0.43\textwidth]{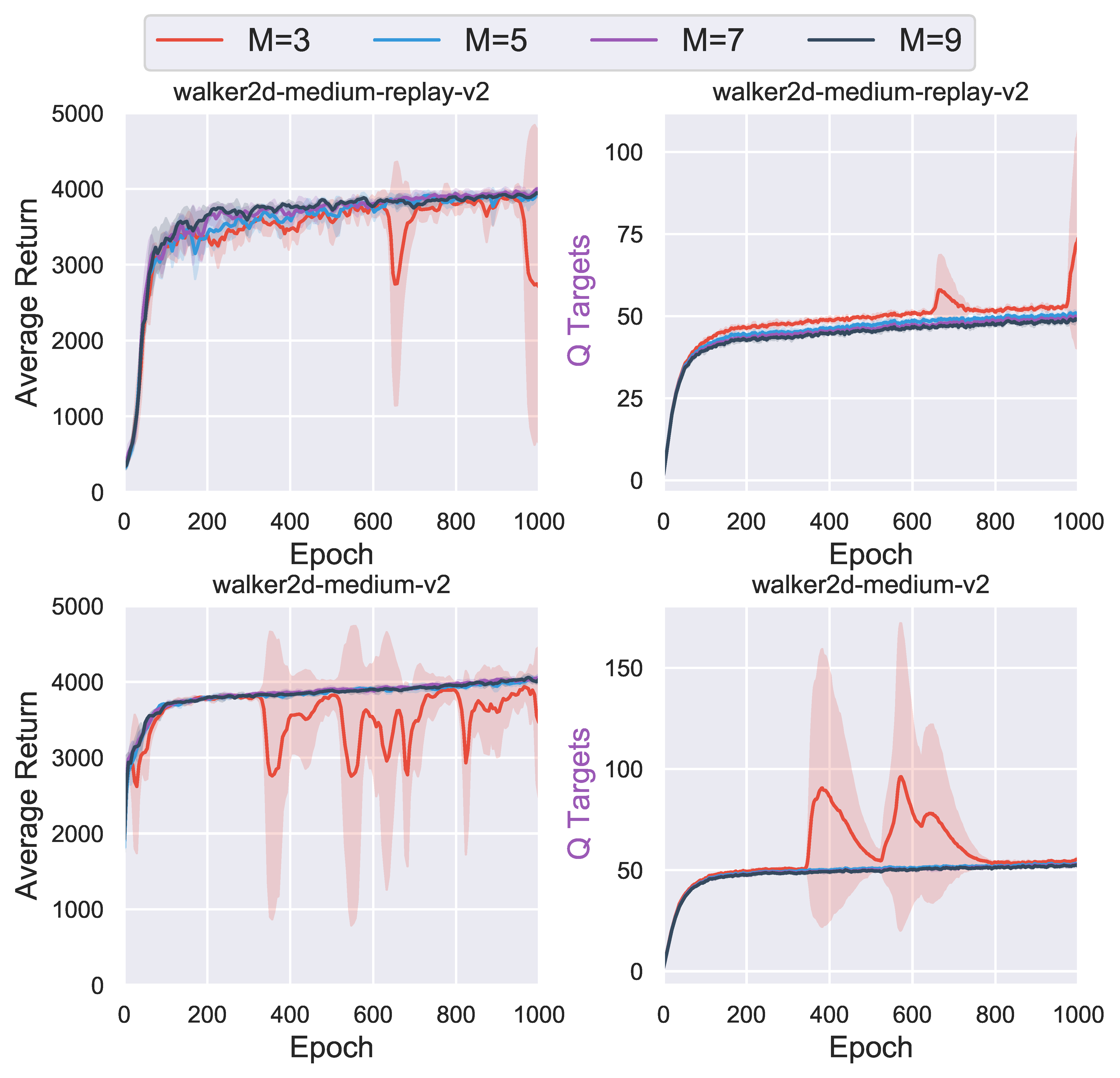}
\end{center}
\vskip -0.2in
\caption{The average return and the Q target of SCORE with different number of ensemble networks.}
\label{fig:hyperparameter_M}
\end{figure}
\begin{figure}[t]
\begin{center}
\includegraphics[width=0.52\textwidth]{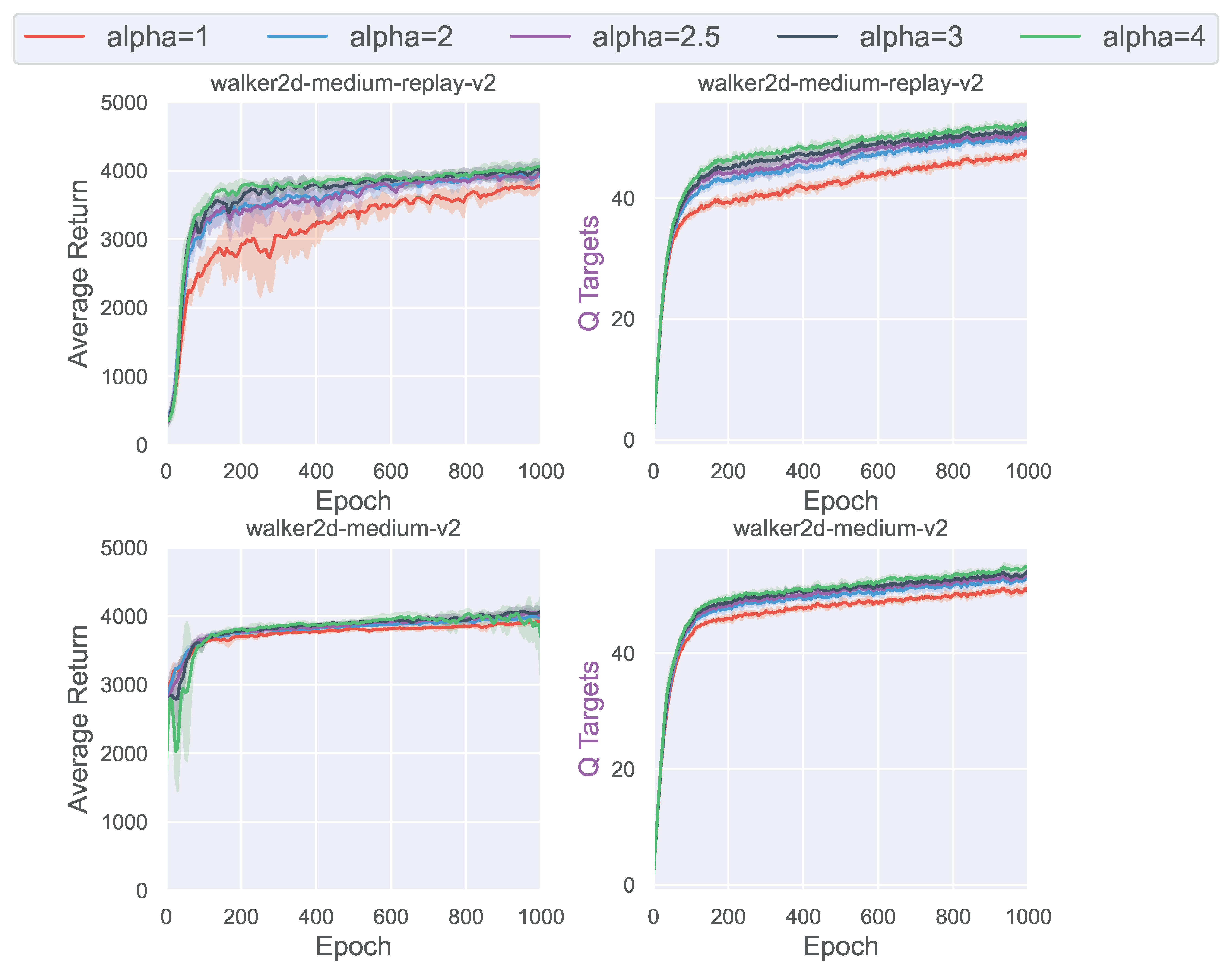}
\end{center}
\vskip -0.2in
\caption{The average return and the Q target of SCORE with different weighting factor $\alpha$ of the Q-loss.}
\label{fig:hyperparameter_alpha}
\end{figure}
\subsection{Hyperparameter Analyzes}
\label{subsec:hyperparameter}
\noindent \textbf{The penalty coefficient $\beta$}. $\beta$ controls the degree of pessimism. In \figurename~\ref{fig:ablation_pe}, we have seen that the agent fails to resolve false correlations without pessimism ($\beta=0$). Here we conduct experiments by choosing $\beta$ from $\{0.1, 0.2, 0.5, 2.0\}$. The experimental results are shown in \figurename~\ref{fig:hyperparameter_beta}. We can see that $\beta=0.1$, $\beta=0.2$ and $\beta=0.5$ work similarly on both datasets, with rapid improvement and stable convergence. In particular, when the penalty gets too large ($\beta=2.0$), the agent becomes over-pessimistic (the Q value rapidly becomes an extremely small negative number). In this case, the agent tends to act conservatively and cannot effectively exploit the information in the dataset to learn informed decisions, resulting in poor performance.

\noindent \textbf{The discount rate of the behavior cloning coefficient \emph{$\gamma_{\text{bc}}$}}. We use a decaying factor $\gamma_{\text{bc}}$ to control the weight of the behavior cloning loss in policy's objective function (\eqref{eq:pi_objective}). We choose discount rate $\gamma_{\text{bc}}$ from $\{0.96, 0.98, 1.0\}$ to validate the sensitivity with respect to behavior cloning. We remark that $\gamma_{\text{bc}}=1$ corresponds to a constant weight. From \figurename~\ref{fig:hyperparameter_gamma} we can see that $\gamma_{\text{bc}}=0.96$ and $\gamma_{\text{bc}}=0.98$ work similarly on the medum-replay dataset, with $\gamma_{\text{bc}}=0.96$ converges faster. In contrast, $\gamma_{\text{bc}}=1.0$ results in a sub-optimal policy that converges prematurely. This is because the medium-replay dataset is generated by a mixture of policies of different quality. A strong behavior cloning regularizer hinders the agent to take the essence and discard the dross. 
On the medium dataset collected by a single behavioral policy, $\gamma_{\text{bc}}=0.96$, $\gamma_{\text{bc}}=0.98$ and $\gamma_{\text{bc}}=1.0$ perform similarly at the initial training stage, with $\gamma_{\text{bc}}=0.98$ displaying the highest level of convergence. Although $\gamma_{\text{bc}}=0.96$ shows the same rapid improvement, it collapses in the later stages of the strategy. In general, a small $\gamma_{\text{bc}}$ is preferable when the dataset is diverse; conversely, a strong constraint is required.

\noindent \textbf{The number of Q-networks}. Empirically, using more ensemble networks provides better uncertainty quantification. We examine the case of $M \in \{3, 5, 7, 9\}$. \figurename~\ref{fig:hyperparameter_M} shows that the bootstrapped ensemble method fails to provide reliable uncertainty estimations when when $M=3$, leading to unstable performance. Interestingly, SCORE works very well with merely five networks. Using more networks improves stability but has little impact on performance improvements. Compared to ~\cite{anUncertaintyBasedOfflineReinforcement2021a} and~\cite{baiPessimisticBootstrappingUncertaintyDriven2022}, the computational overhead is much lower.
% (\rmnum{6}) \emph{Penalty}. \red{red} Equation~\ref{eq:q_objective} computes the penalty $u(s,a)$ based on $Q(s,a)$. We find that using $Q(s', a')$ performs similarly on D4RL-MuJoCo because the reward is deterministic and the uncertainty only comes from state transitions. 

\noindent \textbf{The weighting factor of the Q-loss}. Following Fujimoto et al.~\cite{fujimotoMinimalistApproachOffline2021a}, we employ the Q-value normalization trick to scale the Q-loss during policy optimization. \cite{fujimotoMinimalistApproachOffline2021a} found that their algorithm is robust to the weighting factor $\alpha$ and the recommended setting is $\alpha=2.5$. \figurename~\ref{fig:hyperparameter_alpha} shows a similar observation. SCORE is robust to a wide range of $\alpha$. We use the default value $\alpha=2.5$ in all the other experiments.

% =======================================================
% 
%                    Conclusion 
% 
% =======================================================
\section{Conclusion}
\label{sec:conclusion}
In this work, we propose a novel offline RL algorithm named SCORE (falSe COrrelation REduction), which achieves the SoTA performance with 3.1x acceleration. We identify the false correlation between epistemic uncertainty and decision-making as a core issue in offline RL, which is a broader and more rigorous mathematical concept than the widely studied OOD action problem. To address this issue, SCORE employs the bootstrapped ensemble method to quantify uncertainty and treats it as a penalty. We point out that the failure of previous work is due to the inability to obtain high-quality uncertainty estimations, and propose introducing a simple annealing BC regularizer to solve the problem. Theoretically, we take a step forward from existing work by analyzing policy optimization. We show the proposed algorithm converges to the optimal policy with a sublinear rate under mild assumptions. Moreover, according to the extensive empirical results, SCORE is both provably efficient and practically effective. In the future, we plan to extend the theory to include safety, robustness, and other desiderata and design practical algorithms consistent with the theory.

\ifCLASSOPTIONcaptionsoff
  \newpage
\fi



\bibliography{SCORE_NeurIPS2022}
\bibliographystyle{IEEEtran}


% \end{thebibliography}


\vskip -0.39in
\begin{IEEEbiography}
[{\includegraphics[width=1in,height=1.25in,clip,keepaspectratio]{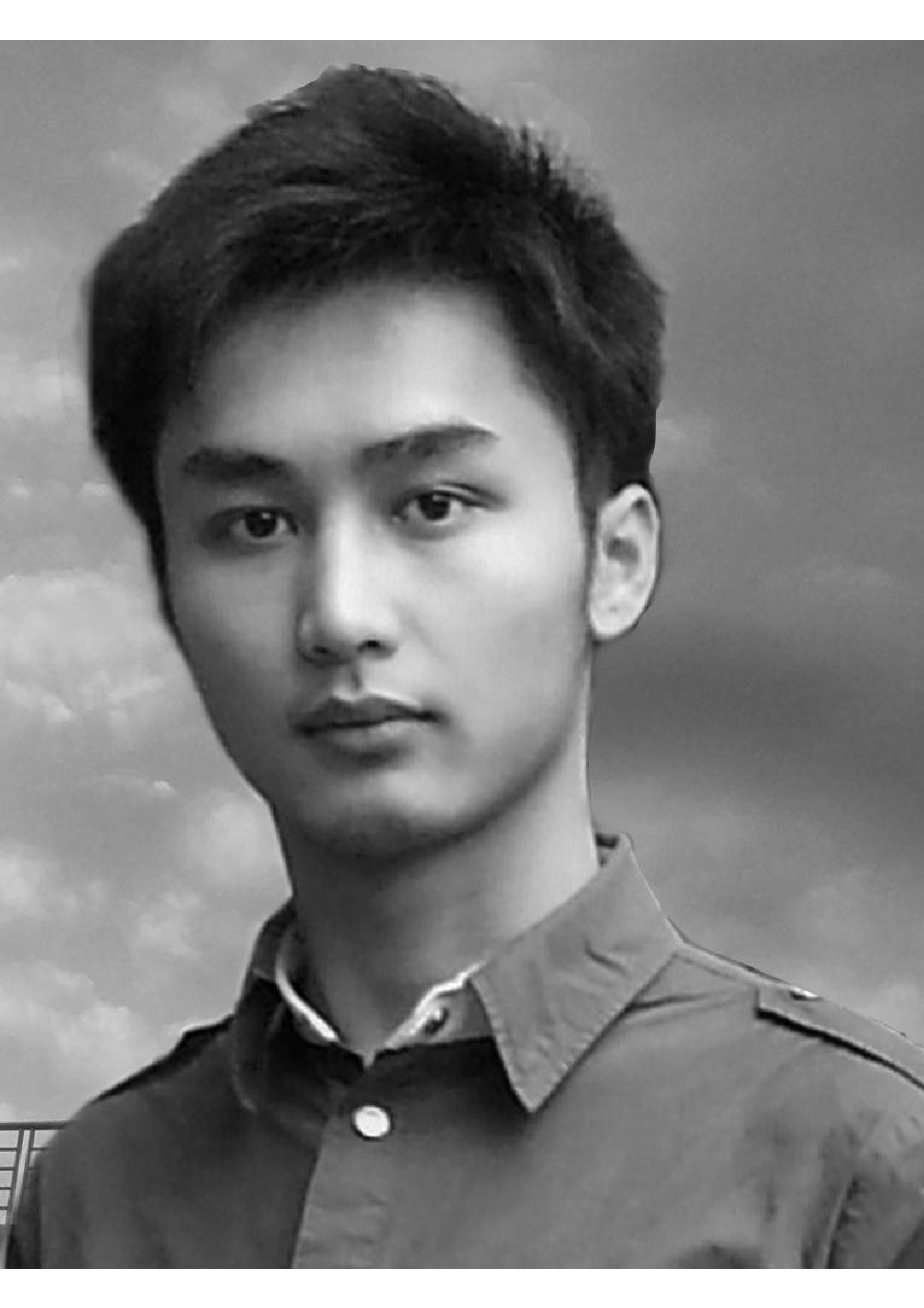}}]
{Zhihong Deng} received his undergraduate and master degree in computer science in 2017 and 2020 respectively from Sun Yat-sen University. {He is pursuing a PhD degree at University of Technology Sydney. His research interests span machine learning and data mining, with a special focus on reinforcement learning.} He has published papers in multiple international conferences and journals, such as AAAI, ICLR, IEEE TCYB and IEEE TNNLS.
\end{IEEEbiography}

\vskip -0.42in
\begin{IEEEbiography}[{\includegraphics[width=1in,height=1.25in,clip,keepaspectratio]{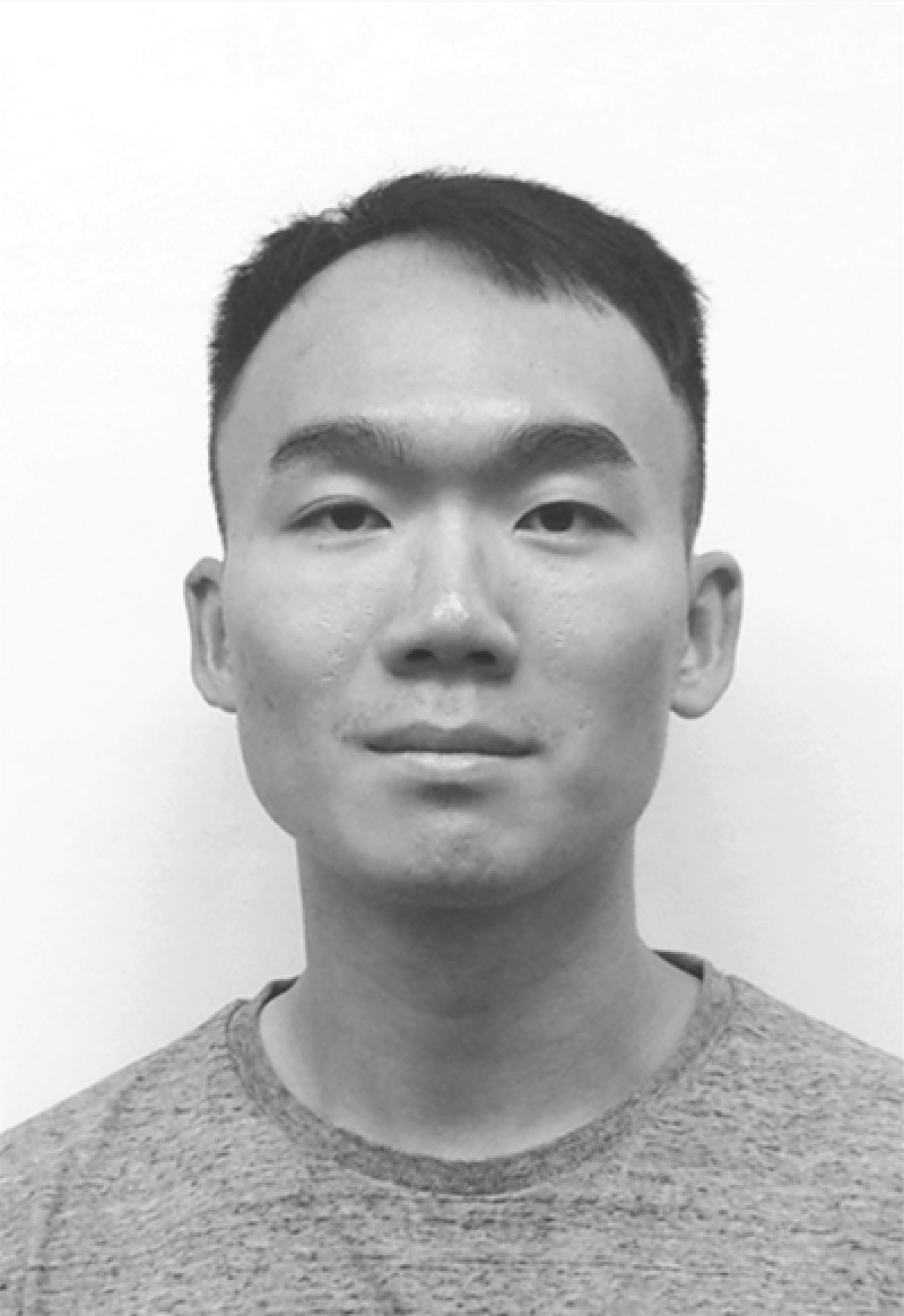}}]{Zuyue Fu} received his Ph.D. degree from Northwestern University in 2022. His research interest lies in the intersection of optimization and machine learning, with a special focus on reinforcement learning. He is now a research scientist with Meta.
\end{IEEEbiography}

\vskip -0.36in
\begin{IEEEbiography}[{\includegraphics[width=1in,height=1.25in,clip,keepaspectratio]{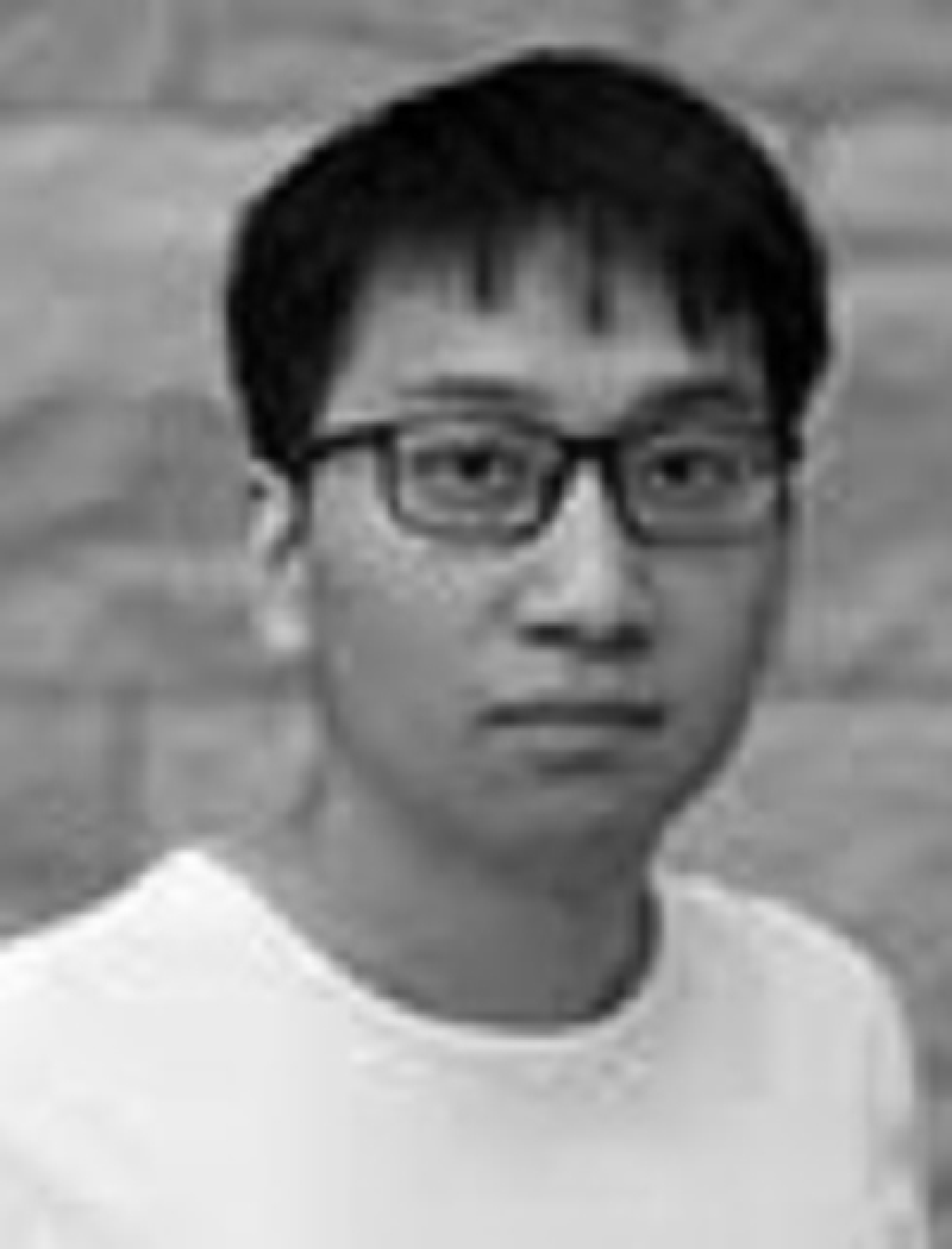}}]{Lingxiao Wang} received his undergraduate degree in mathematics in 2017 from Nanyang Technological University. He received his master's and Ph.D. degrees from Northwestern University in 2020 and 2022, respectively. His research interests lie in reinforcement learning, with a special focus on exploration and sample efficiency analysis.
\end{IEEEbiography}

\begin{IEEEbiography}[{\includegraphics[width=1in,height=1.25in,clip,keepaspectratio]{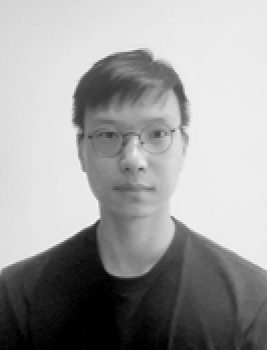}}]{Zhuoran Yang} is an Assistant Professor of Statistics and Data Science at Yale University, starting in July 2022. His research interests lie in the interface between machine learning, statistics, and optimization. He is particularly interested in the foundations of reinforcement learning, representation learning, and deep learning. Before joining Yale, Zhuoran worked as a postdoctoral researcher at the University of California, Berkeley, advised by Michael. I. Jordan. Prior to that, he obtained his Ph.D. from the Department of Operations Research and Financial Engineering at Princeton University, co-advised by Jianqing Fan and Han Liu. He received his bachelor’s degree in Mathematics from Tsinghua University in 2015.
\end{IEEEbiography}

\begin{IEEEbiography}[{\includegraphics[width=1in,height=1.25in,clip,keepaspectratio]{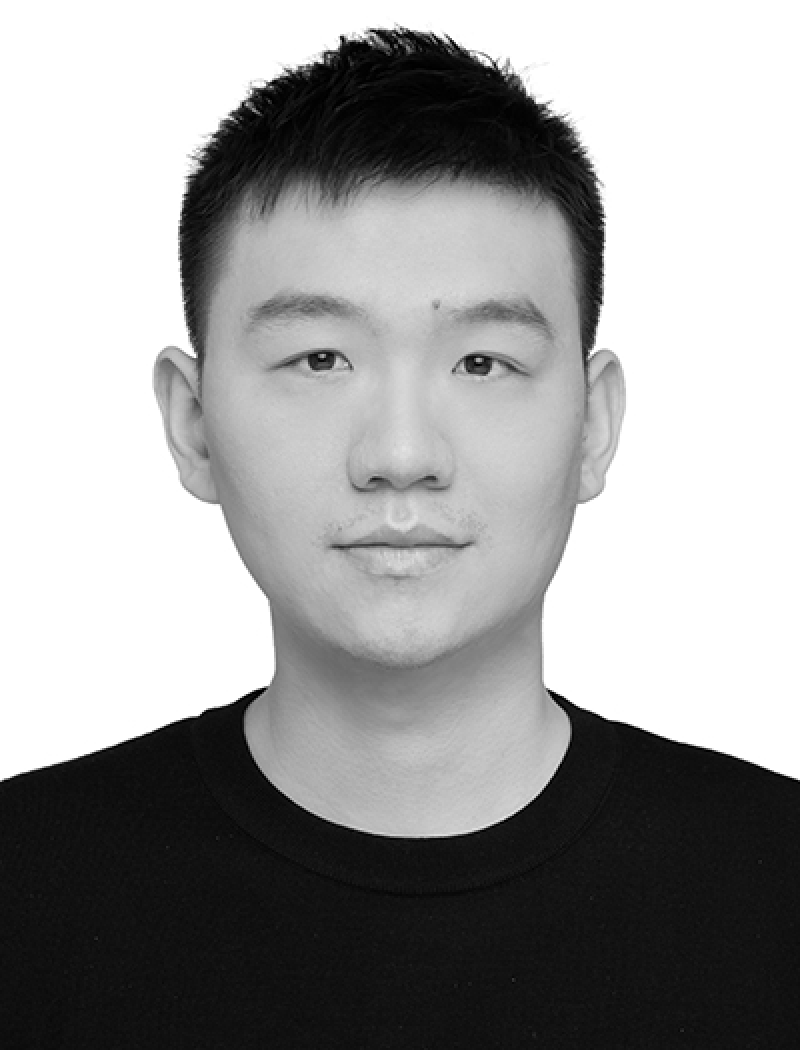}}]{Chenjia Bai} received the B.S., M.S., and Ph.D. degrees in Computer Science and Technology from the Harbin Institute of Technology, Harbin, China, in 2015, 2017, and 2022, respectively. He is currently a Researcher with Shanghai Artificial Intelligence Laboratory. His main research interests include reinforcement learning, deep learning networks, and robotics.
\end{IEEEbiography}

\begin{IEEEbiography}[{\includegraphics[width=1in,height=1.25in,clip,keepaspectratio]{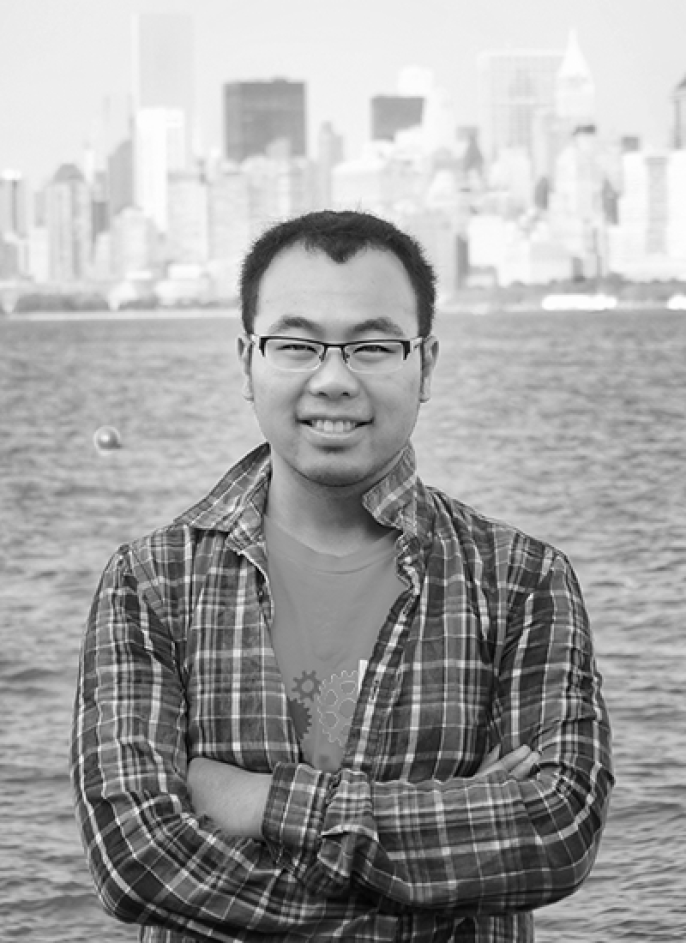}}]{Tianyi Zhou} is a tenure-track assistant professor of computer science at the University of Maryland, College Park. He received his Ph.D. from the school of computer science \& engineering at the University of Washington, Seattle. His research interests are in machine learning, optimization, and natural language processing (NLP). 
% His recent works study curriculum learning that can combine high-level human learning strategies with model training dynamics to create a hybrid intelligence. 
He published over 70 papers and is a recipient of the Best Student Paper Award at ICDM 2013 and the 2020 IEEE Computer Society TCSC Most Influential Paper Award.
\end{IEEEbiography}

\begin{IEEEbiography}[{\includegraphics[width=0.9in,height=1.35in,clip,keepaspectratio]{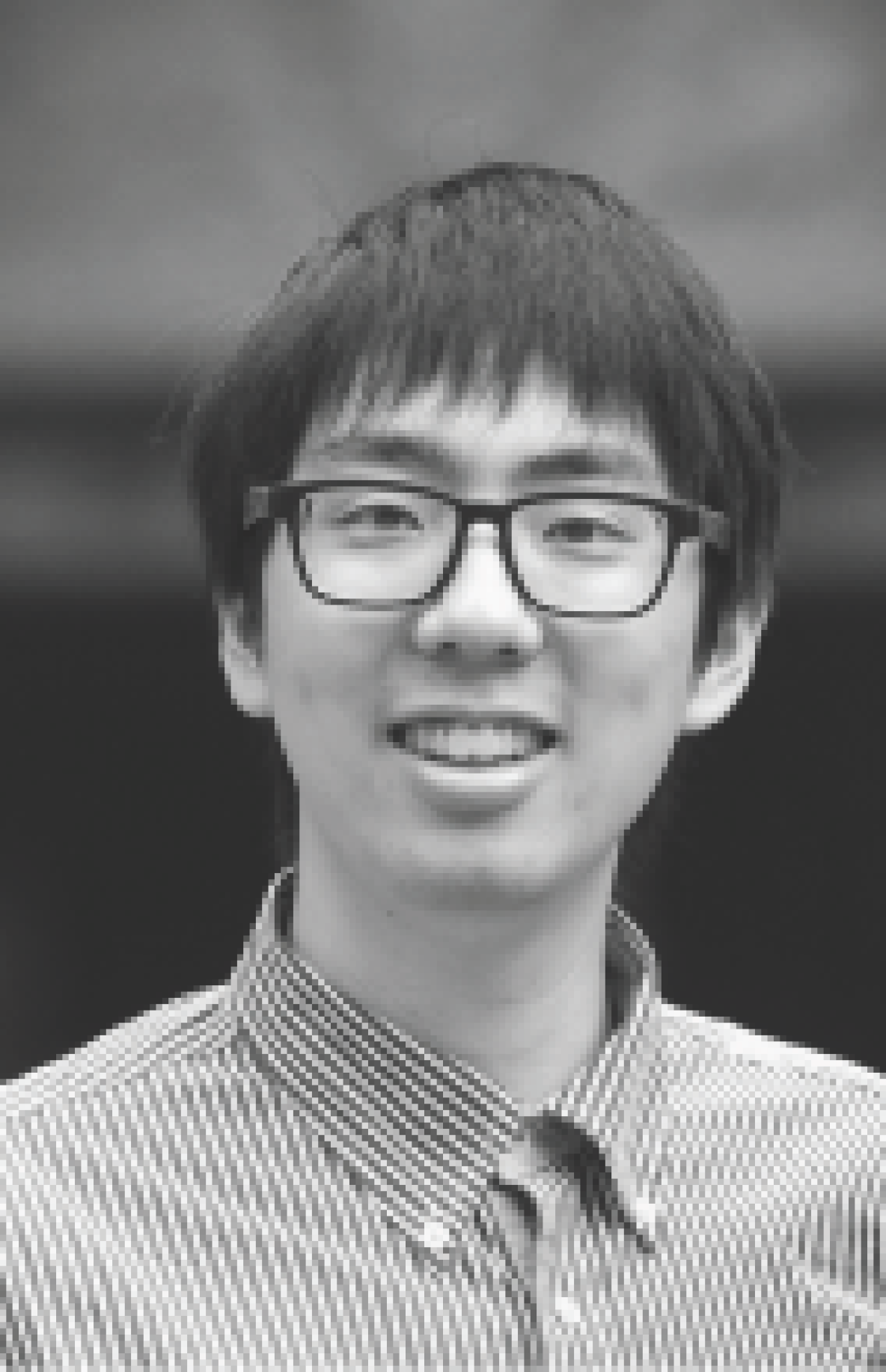}}]{Zhaoran Wang} is an assistant professor at Northwestern University, working at the interface of machine learning, statistics, and optimization. He is the recipient of the AISTATS (Artificial Intelligence and Statistics Conference) notable paper award, ASA (American Statistical Association) best student paper in statistical learning and data mining, INFORMS (Institute for Operations Research and the Management Sciences) best student paper finalist in data mining, Microsoft Ph.D. Fellowship, Simons-Berkeley/J.P. Morgan AI Research Fellowship, Amazon Machine Learning Research Award, and NSF CAREER Award.
\end{IEEEbiography}

\begin{IEEEbiography}[{\includegraphics[width=0.9in,height=1.35in,clip,keepaspectratio]{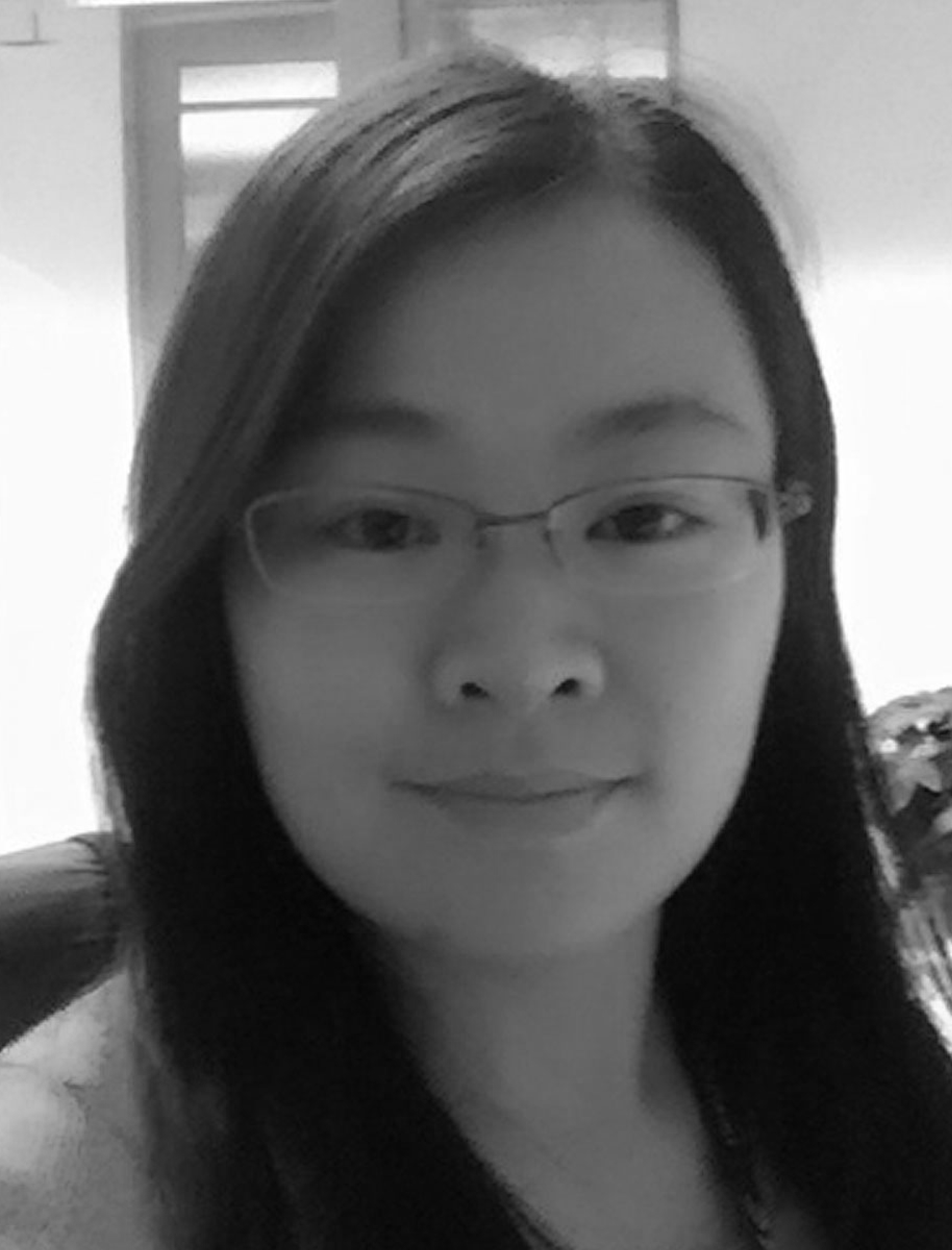}}]{Jing Jiang} is a Senior Lecturer in the School of Computer Science, a core member of Australian Artificial Intelligence Institute (AAII), at the University of Technology Sydney (UTS) in Australia. Her research interests focus on machine learning and its applications. She is the recipient of the DECRA (Discovery Early Career Researcher Award) fellowship funded by ARC (Australian Research Council). She has published over 70 papers in the related areas of AI in the top-tier conferences and journals, such as NeurIPS, ICML, ICLR, AAAI, IJCAI, KDD, TNNLS and TKDE.
\end{IEEEbiography}

%\newpage



%\vfill


\newpage
\appendices
\onecolumn
% =======================================================
% 
%                    Proofs
% 
% =======================================================
\section{Theoretical Analyzes and Proofs}
\label{app:theory}
\subsection{Proof of Theorem \ref{thm:main}}
\label{sec:main}
\begin{proof}

We denote by
\#\label{eq:sumsub}
\text{AveSubOptGap}(K) = \frac{1}{K} \cdot \sum_{k=0}^{K-1} \bigl (V_k^*(s_0) - V_{k}^{\pi_k}(s_0) \bigr ). 
\#
By the definition of $\regr(K)$ in \eqref{eq:regr}, we know that $\regr(K) \leq \text{AveSubOptGap}(K)$. 

Before we prove the theorem, we first introduce the following useful lemmas. 

\begin{lemma}[Suboptimality Decomposition]\label{lemma:reg-decomp}
For $\text{AveSubOptGap}(K)$ defined in \eqref{eq:sumsub}, we have
\$
\text{AveSubOptGap}(K) = \frac{1}{K} \cdot \sum_{k=0}^{K-1} \sum_{t = 0}^\infty \gamma^t \cdot & \Bigl( \EE_{\pi^*}\Bigl[ \Bigl\lan Q_k(s_t, \cdot) - \lambda_k \cdot \log \frac{\pi_k(\cdot\given s_t)}{\pi_0(\cdot \given s_t)}, \pi^*(\cdot\given s_t) - \pi_k(\cdot\given s_t) \Bigr \ran \Biggiven s_0\Bigr]\notag\\
&\quad  + \EE_{\pi^*}\bigl[\iota_k(s_t, a_t) \given s_0 \bigr] - \EE_{\pi_k}\bigl[\iota_k(s_t, a_t) \given s_0 \bigr]\Bigr), 
\$
where $\iota_k(s,a) = r(s,a) + \gamma \cdot  \EE_{s' \sim \cP(\cdot \given s, a)}[V_k(s')] - Q_k(s,a)$ for any $(s,a)\in \cS\times \cA$. 
\end{lemma}
\begin{proof}
See proof of Lemma 4.2 in \cite{caiProvablyEfficientExploration2020} for a detailed proof.
\end{proof}

\begin{lemma}[Policy Improvement]\label{lemma:pi}
It holds for any $k$ that
\$
& (\eta_k + \lambda_k)^{-1} \cdot \Bigl\lan Q_k(s_t, \cdot) - \lambda_k \cdot \log \frac{\pi_k(\cdot\given s_t)}{\pi_0(\cdot \given s_t)}, \pi^*(\cdot\given s_t) - \pi_k(\cdot\given s_t) \Bigr \ran \notag \\
& \qquad \leq \kl \bigl( \pi^*(\cdot \given s_t) \| \pi_k(\cdot \given s_t) \bigr) - \kl \bigl( \pi^*(\cdot \given s_t) \| \pi_{k+1}(\cdot \given s_t) \bigr) \\
& \qquad \qquad + (\eta_k + \lambda_k)^{-2} \cdot \bigl (1 + \lambda_k \cdot \alpha^4(1 - \alpha)^{-4} \bigr)^2 \cdot (1 - \gamma)^{-2}. 
\$
\end{lemma}
\begin{proof}
See Section \ref{sec:pi} for a detailed proof. 
\end{proof}

\begin{lemma}[Pessimism]\label{lemma:pess}
Under Assumption \ref{ass:pess}, with probability at least $1 - \xi$, it holds for any $(s,a,k) \in \cS \times\cA\times [K]$ that 
\$
0 \leq \iota_k(s,a) \leq 2U(s,a), 
\$
where $\iota_k = \mathcal{B}Q_k - \widehat{\mathcal{B}}Q_k$ is the epistemic error defined in \eqref{eq:epistemic_error}. 

\end{lemma}
\begin{proof}
See proof of Lemma 5.1 in \cite{jinPessimismProvablyEfficient2021a} for a detailed proof. 
\end{proof}

Now we prove the theorem. 
By Lemma \ref{lemma:reg-decomp}, we have
\#\label{eq:jjj1}
\text{AveSubOptGap}(K) & = \frac{1}{K} \cdot \sum_{k=0}^K \sum_{t = 0}^\infty \gamma^t \cdot   \Bigl( \EE_{\pi^*}\Bigl[ \Bigl\lan Q_k(s_t, \cdot) - \lambda_k \cdot \log \frac{\pi_k(\cdot\given s_t)}{\pi_0(\cdot \given s_t)}, \pi^*(\cdot\given s_t) - \pi_k(\cdot\given s_t) \Bigr \ran \Biggiven s_0\Bigr]\notag\\
&\qquad\qquad\qquad\quad + \EE_{\pi^*}\bigl[\iota_k(s_t, a_t) \given s_0 \bigr] - \EE_{\pi_k}\bigl[\iota_k(s_t, a_t) \given s_0 \bigr]\Bigr)\notag \\
& \leq \frac{1}{K} \cdot \sum_{k=0}^K \sum_{t = 0}^\infty \gamma^t \cdot \Bigl( \EE_{\pi^*} \bigl[ \eta \cdot \kl \bigl( \pi^*(\cdot \given s_t) \| \pi_k(\cdot \given s_t) \bigr) - \eta\cdot \kl \bigl( \pi^*(\cdot \given s_t) \| \pi_{k+1}(\cdot \given s_t) \bigr) \bigr] \notag \\
& \qquad \qquad \qquad \quad  + \eta^{-1} \cdot \bigl (1 + \lambda_k \cdot \alpha^4(1 - \alpha)^{-4} \bigr)^2 \cdot (1 - \gamma)^{-2} \notag \\
&\qquad\qquad\qquad\quad + \EE_{\pi^*}\bigl[\iota_k(s_t, a_t) \given s_0 \bigr] - \EE_{\pi_k}\bigl[\iota_k(s_t, a_t) \given s_0 \bigr]\Bigr),
\#
where we denote by $\eta = \eta_k + \lambda_k$, and the last inequality comes from Lemma \ref{lemma:pi}. Further, by telescoping the sum of $k$ on the right-hand side of \eqref{eq:jjj1} and the non-negativity of the KL divergence, it holds with probability at least $1 - \xi$ that
\#\label{eq:jjj2}
\text{AveSubOptGap}(K) & \leq \frac{\eta}{K} \cdot \sum_{t = 0}^\infty \gamma^t \cdot  \EE_{\pi^*} \bigl[ \kl \bigl( \pi^*(\cdot \given s_t) \| \pi_0(\cdot \given s_t) \bigr) \bigr] +   \eta^{-1} \cdot \bigl (1 + \alpha^4(1 - \alpha)^{-4} \bigr)^2 \cdot (1 - \gamma)^{-3} \notag \\
& \qquad + \frac{1}{K} \cdot \sum_{k=0}^K \sum_{t = 0}^\infty \gamma^t \cdot \Bigl( \EE_{\pi^*}\bigl[\iota_k(s_t, a_t) \given s_0 \bigr] - \EE_{\pi_k}\bigl[\iota_k(s_t, a_t) \given s_0 \bigr] \Bigr) \notag \\
& \leq \frac{\eta}{K} \cdot \sum_{t = 0}^\infty \gamma^t \cdot  \EE_{\pi^*} \bigl[ \kl \bigl( \pi^*(\cdot \given s_t) \| \pi_0(\cdot \given s_t) \bigr) \bigr] +  \eta^{-1} \cdot \bigl (1 + \alpha^4(1 - \alpha)^{-4} \bigr)^2 \cdot (1 - \gamma)^{-3} \notag \\
& \qquad + \sum_{t = 0}^\infty \gamma^t \cdot \EE_{\pi^*}\bigl[2 U (s_t, a_t) \given s_0 \bigr], 
\#
where the last inequality comes from Lemma \ref{lemma:pess}. Now, by taking $\eta = \sqrt{\zeta / K}$, where
\$
\zeta =  \bigl (1 + \alpha^4(1 - \alpha)^{-4} \bigr)^2 \cdot \sum_{t = 0}^\infty \gamma^t \cdot  \EE_{\pi^*} \bigl[ \kl \bigl( \pi^*(\cdot \given s_t) \| \pi_0(\cdot \given s_t) \bigr) \bigr], 
\$
combining \eqref{eq:jjj2}, with probability at least $1 - \xi$, we have
\$
\text{AveSubOptGap}(K) = O\bigl((1 - \gamma)^{-3} \sqrt{\zeta / K} \bigr) + \perr. 
\$
Here $\perr$ is the intrinsic uncertainty defined in \eqref{eq:perr}. By the fact that $\regr(K) \leq \text{AveSubOptGap}(K)$, we conclude the proof. 
\end{proof}

% =======================================================
% 
% =======================================================
\subsection{Proof of Lemmas}
\subsubsection{Proof of Lemma \ref{lemma:dpg=ppo}}\label{sec:dpg=ppo}
\begin{proof}
By plugging the definition of $\Pi_\phi(s) = \EE_{a\sim \pi_\phi(\cdot \given s)}[a]$ and the linear parameterization $Q_k(s,a) = \theta_k(s)^\top a$ into \eqref{eq:ppo-obj}, we have
\#\label{eq:ppo2}
\cL^k_{\rm OPO}(\phi) = \EE_{s\sim \cD} \bigl[ Q_k(s, \Pi_\phi(s)) - \lambda_k \cdot \kl (\pi_{\phi}(\cdot \given s) \| \pi_0(\cdot \given s) )  - \eta_k\cdot \kl (\pi_\phi (\cdot\given s) \| \pi_k(\cdot \given s)) \bigr].
\#
It holds for any $s\in \cS$ that
\#\label{eq:kl1}
\nabla_\phi \kl (\pi_\phi(\cdot\given s) \| \pi_0(\cdot \given s)) & = \nabla_\phi \EE_{a \sim \pi_\phi(\cdot \given s)} \Bigl[\log \bigl(  \pi_\phi(a \given s) / \pi_0(a \given s)   \bigr) \Bigr] \notag \\
& = \nabla_\phi \EE_{a \sim \pi_\phi(\cdot \given s)} \bigl[ (\phi - \phi_0)^\top \psi(s,a) + Z_\phi(s) - Z_{\phi_0}(s) \bigr]\notag \\
& =  \nabla_\phi \EE_{a \sim \pi_\phi(\cdot \given s)} [\psi(s,a)] (\phi - \phi_0) + \EE_{a \sim \pi_\phi(\cdot \given s)} [\psi(s,a)] - \nabla_\phi Z_\phi(s) \notag \\
& =  \nabla_\phi^2 Z_\phi(s) (\phi - \phi_0) = \var_{a\sim \pi_\phi(\cdot\given s)}[\psi(s,a)] (\phi - \phi_0)\notag \\
& = I_\phi(s) (\phi - \phi_0). 
\#
Similarly, we have
\#\label{eq:kl2}
\nabla_\phi \kl (\pi_\phi(\cdot\given s) \| \pi_k(\cdot \given s)) = I_\phi(s) (\phi - \phi_k). 
\#
for any $s\in \cS$. Thus, by combining \eqref{eq:ppo2}, \eqref{eq:kl1}, and \eqref{eq:kl2}, the stationary point $\phi_{k+1}$ of $\cL^k_{\rm OPO}(\phi)$ satisfies
\#\label{eq:station1}
\EE_{s\sim \cD} \Bigl[ \nabla_a Q_k(s, \Pi_{\phi_{k+1}}(s)) \nabla_\phi \Pi_{\phi_{k+1}}(s)  - \lambda_k \cdot I_{\phi_{k+1}}(s) (\phi_{k+1} - \phi_0)\notag \\
- \eta_k \cdot I_{\phi_{k+1}}(s) (\phi_{k+1} - \phi_k)   \Bigr] = 0. 
\#
Now, by \eqref{eq:station1}, we have
\$
\phi_{k+1} = \frac{\eta_k \phi_k + \lambda_k \phi_0}{\eta_k + \lambda_k} + (\eta_k + \lambda_k)^{-1} \cdot I_{\phi_{k+1}}^{-1} \EE_{s\sim \cD} \bigl[ \nabla_a Q_k(s, \Pi_{\phi_{k+1}}(s)) \nabla_\phi \Pi_{\phi_{k+1}}(s) \bigr], 
\$
which concludes the proof. 
\end{proof}

\subsubsection{Proof of Lemma \ref{lemma:pi}}
\begin{proof} \label{sec:pi}

First, by maximizing \eqref{eq:ppo-obj}, we have
\$
\pi_{k+1}(a\given s)\propto \exp\{ (\eta_k + \lambda_k)^{-1} \cdot (Q_k(s,a) + \eta_k f_k(s,a) + \lambda_k f_0(s,a)) \}. 
\$
Thus, for any policy $\pi'$ and $\pi''$, it holds for any $s\in \cS$ that
\#\label{eq:fff1}
& \Bigl \lan \log \frac{\pi_{k+1}(\cdot \given s)}{\pi_k(\cdot\given s)} , \pi'(\cdot\given s) - \pi''(\cdot\given s) \Bigr\ran\notag\\
& \qquad = (\eta_k + \lambda_k)^{-1} \cdot \Bigl \lan Q_k(s,\cdot) - \lambda_k \cdot \log \frac{\pi_k(\cdot\given s)}{\pi_0(\cdot\given s)}, \pi'(\cdot\given s) - \pi''(\cdot\given s) \Bigr \ran. 
\#
We will use \eqref{eq:fff1} in the following proof.

Note that 
\#\label{eq:pi1}
& \kl \bigl( \pi^*(\cdot \given s_t) \| \pi_k(\cdot \given s_t) \bigr) - \kl \bigl( \pi^*(\cdot \given s_t) \| \pi_{k+1}(\cdot \given s_t) \bigr) \notag \\
& \qquad =  \Bigl \lan \log\frac{\pi_{k+1}(\cdot \given s_t)}{\pi_k(\cdot \given s_t) } , \pi^*(\cdot \given s_t)  \Bigr\ran \notag \\
& \qquad = \Bigl \lan \log\frac{\pi_{k+1}(\cdot \given s_t)}{\pi_k(\cdot \given s_t) }, \pi^*(\cdot \given s_t) - \pi_{k+1}(\cdot \given s_t)  \Bigr\ran + \kl \bigl( \pi_{k+1}(\cdot \given s_t) \| \pi_k(\cdot \given s_t) \bigr). 
\#
In the meanwhile, we have
\#\label{eq:pi2}
& \Bigl \lan \log\frac{\pi_{k+1}(\cdot \given s_t)}{\pi_k(\cdot \given s_t) }, \pi^*(\cdot \given s_t) - \pi_{k+1}(\cdot \given s_t)  \Bigr\ran \notag\\
& \qquad = \Bigl \lan \log\frac{\pi_{k+1}(\cdot \given s_t)}{\pi_k(\cdot \given s_t) }, \pi^*(\cdot \given s_t) - \pi_k(\cdot \given s_t)  \Bigr\ran + \Bigl \lan \log\frac{\pi_{k+1}(\cdot \given s_t)}{\pi_k(\cdot \given s_t) }, \pi_k(\cdot \given s_t) - \pi_{k+1}(\cdot \given s_t)  \Bigr\ran \notag \\
& \qquad  =  (\eta_k + \lambda_k)^{-1} \cdot   \Bigl \lan Q_k(s_t, \cdot ) - \lambda_k \cdot \log \frac{\pi_k(\cdot \given s_t)}{\pi_0(\cdot \given s_t)}, \pi^*(\cdot \given s_t) - \pi_k(\cdot \given s_t)  \Bigr\ran \notag \\
& \qquad \qquad + (\eta_k + \lambda_k)^{-1} \cdot   \Bigl \lan Q_k(s_t, \cdot ) - \lambda_k \cdot \log \frac{\pi_k(\cdot \given s_t)}{\pi_0(\cdot \given s_t)}, \pi_k(\cdot \given s_t) - \pi_{k+1}(\cdot \given s_t)  \Bigr\ran, 
\#
where the last equality comes from \eqref{eq:fff1}. 
Combining \eqref{eq:pi1} and \eqref{eq:pi2}, we have
\#\label{eq:pi3}
& (\eta_k + \lambda_k)^{-1} \cdot   \Bigl \lan Q_k(s_t, \cdot ) - \lambda_k \cdot \log \frac{\pi_k(\cdot \given s_t)}{\pi_0(\cdot \given s_t)}, \pi^*(\cdot \given s_t) - \pi_k(\cdot \given s_t)  \Bigr\ran\notag \\
& \qquad = \kl \bigl( \pi^*(\cdot \given s_t) \| \pi_k(\cdot \given s_t) \bigr) - \kl \bigl( \pi^*(\cdot \given s_t) \| \pi_{k+1}(\cdot \given s_t) \bigr) - \kl \bigl( \pi_{k+1}(\cdot \given s_t) \| \pi_k(\cdot \given s_t) \bigr) \notag \\
& \qquad \qquad - (\eta_k + \lambda_k)^{-1} \cdot   \Bigl \lan Q_k(s_t, \cdot ) - \lambda_k \cdot \log \frac{\pi_k(\cdot \given s_t)}{\pi_0(\cdot \given s_t)}, \pi_k(\cdot \given s_t) - \pi_{k+1}(\cdot \given s_t)  \Bigr\ran \notag \\
& \qquad \leq \kl \bigl( \pi^*(\cdot \given s_t) \| \pi_k(\cdot \given s_t) \bigr) - \kl \bigl( \pi^*(\cdot \given s_t) \| \pi_{k+1}(\cdot \given s_t) \bigr) - \bigl \| \pi_{k+1}(\cdot \given s_t) - \pi_k(\cdot \given s_t) \bigr\|_1^2 / 2 \notag \\
& \qquad \qquad - (\eta_k + \lambda_k)^{-1} \cdot   \Bigl \lan Q_k(s_t, \cdot ) - \lambda_k \cdot \log \frac{\pi_k(\cdot \given s_t)}{\pi_0(\cdot \given s_t)}, \pi_k(\cdot \given s_t) - \pi_{k+1}(\cdot \given s_t)  \Bigr\ran, 
\#
where the last inequality comes from Pinsker's inequality. 
To upper bound the last term on the right-hand side of \eqref{eq:pi3}, we characterize $\log(\pi_k(a\given s) / \pi_0(a\given s))$ as follows. 

\noindent\textbf{Characterization of $\log(\pi_k / \pi_0)$.} For any $(s,a) \in \cS\times\cA$, we have
\$
\log\frac{\pi_k}{\pi_0} & = \log\Bigl(\frac{\pi_k}{\pi_{k-1}} \cdot \frac{\pi_{k-1}}{\pi_{k-2}} \cdot \cdots \cdot \frac{\pi_{1}}{\pi_{0}} \Bigr) = \sum_{i = 0}^{k-1} \log\frac{\pi_{i+1}}{\pi_{i}}. 
\$
Then, we have 
\#\label{eq:re1}
\log\frac{\pi_k}{\pi_0}  & = \sum_{i = 0}^{k-1}  \log\frac{\pi_{i+1}}{\pi_{i}}  =  \sum_{i = 0}^{k-1} \Bigl ( Q_i + \lambda_i \cdot \log\frac{\pi_i}{\pi_0}  \Bigr ) + Z_1,
\#
where $Z_1$ is a function independent of $a$. 
Now, by recursively applying \eqref{eq:re1}, we have
\#\label{eq:chara-log}
\log\frac{\pi_k}{\pi_0} = \sum_{i = 0}^{k-1} Q_i \cdot \sum_{j = i+1}^k \lambda_j \Bigl (1 + \sum_{\ell = 0}^{k-j-1} \varepsilon_\ell \prod_{p = 0}^{\ell} \lambda_{k-p} \Bigr ) + Z_2, 
\#
where $\varepsilon_\ell$ is either $1$ or $-1$, and $Z_2$ is a function independent of $a$. 

Now, by \eqref{eq:chara-log}, the last term on the right-hand side of \eqref{eq:pi3} can be upper bounded as follows, 
\#\label{eq:bound-q-log}
& - (\eta_k + \lambda_k)^{-1} \cdot   \Bigl \lan Q_k(s_t, \cdot ) - \lambda_k \cdot \log \frac{\pi_k(\cdot \given s_t)}{\pi_0(\cdot \given s_t)}, \pi_k(\cdot \given s_t) - \pi_{k+1}(\cdot \given s_t)  \Bigr\ran \notag \\
& \qquad = - (\eta_k + \lambda_k)^{-1} \cdot   \Bigl \lan Q_k(s_t, \cdot ) - \lambda_k  \sum_{i = 0}^{k-1} Q_i(s_t, \cdot)  \sum_{j = i+1}^k \lambda_j \Bigl(1 + \sum_{\ell = 0}^{k-j-1} \varepsilon_\ell \prod_{p = 0}^{\ell} \lambda_{k-p} \Bigr) -\lambda_k \cdot Z_2(s_t) , \notag \\
& \qquad \qquad \qquad \qquad \qquad \qquad \pi_k(\cdot \given s_t) - \pi_{k+1}(\cdot \given s_t)  \Bigr\ran \notag  \\
& \qquad = - (\eta_k + \lambda_k)^{-1} \cdot   \Bigl \lan Q_k(s_t, \cdot ) - \lambda_k  \sum_{i = 0}^{k-1} Q_i(s_t, \cdot)  \sum_{j = i+1}^k \lambda_j \Bigl(1 + \sum_{\ell = 0}^{k-j-1} \varepsilon_\ell \prod_{p = 0}^{\ell} \lambda_{k-p} \Bigr), \notag \\
& \qquad \qquad \qquad \qquad \qquad \qquad \pi_k(\cdot \given s_t) - \pi_{k+1}(\cdot \given s_t)  \Bigr\ran \notag \\ 
& \qquad \leq (\eta_k + \lambda_k)^{-1} \cdot \Bigl\| Q_k(s_t, \cdot ) - \lambda_k  \sum_{i = 0}^{k-1} Q_i(s_t, \cdot)  \sum_{j = i+1}^k \lambda_j \Bigl(1 + \sum_{\ell = 0}^{k-j-1} \varepsilon_\ell \prod_{p = 0}^{\ell} \lambda_{k-p} \Bigr) \Bigr\|_\infty  \\
& \qquad \qquad\qquad \qquad \quad \cdot \|\pi_k(\cdot \given s_t) - \pi_{k+1}(\cdot \given s_t)  \|_1 \notag, 
\#
where the last line comes from H\"older's inequality. In the meanwhile, it holds that
\#\label{eq:q-bound}
\|Q_k\|_\infty \leq (1 - \gamma)^{-1},
\#
and 
\#\label{eq:log-bound}
& \Bigl\| \sum_{i = 0}^{k-1} Q_i(s_t, \cdot)  \sum_{j = i+1}^k \lambda_j \Bigl(1 + \sum_{\ell = 0}^{k-j-1} \varepsilon_\ell \prod_{p = 0}^{\ell} \lambda_{k-p} \Bigr) \Bigr\|_\infty  \leq \alpha^4(1 - \alpha)^{-4}(1 - \gamma)^{-1}. 
\#
Now, by plugging \eqref{eq:q-bound} and \eqref{eq:log-bound} into \eqref{eq:bound-q-log}, we have
\#\label{eq:pi4}
& (\eta_k + \lambda_k)^{-1} \cdot   \Bigl \lan Q_k(s_t, \cdot ) - \lambda_k \cdot \log \frac{\pi_k(\cdot \given s_t)}{\pi_0(\cdot \given s_t)}, \pi_k(\cdot \given s_t) - \pi_{k+1}(\cdot \given s_t)  \Bigr\ran \notag \\
& \qquad \leq (\eta_k + \lambda_k)^{-1} \cdot (1 + \lambda_k \alpha^4(1 - \alpha)^{-4})(1 - \gamma)^{-1} \cdot \| \pi_k(\cdot \given s_t) - \pi_{k+1}(\cdot \given s_t)\|_1. 
\#
Now, combining \eqref{eq:pi3} and \eqref{eq:pi4}, it holds that 
\$
& (\eta_k + \lambda_k)^{-1} \cdot   \Bigl \lan Q_k(s_t, \cdot ) - \lambda_k \cdot \log \frac{\pi_k(\cdot \given s_t)}{\pi_0(\cdot \given s_t)}, \pi^*(\cdot \given s_t) - \pi_k(\cdot \given s_t)  \Bigr\ran\notag \\
& \qquad \leq \kl \bigl( \pi^*(\cdot \given s_t) \| \pi_k(\cdot \given s_t) \bigr) - \kl \bigl( \pi^*(\cdot \given s_t) \| \pi_{k+1}(\cdot \given s_t) \bigr) - \bigl \| \pi_{k+1}(\cdot \given s_t) - \pi_k(\cdot \given s_t) \bigr\|_1^2 / 2 \notag \\
& \qquad \qquad + (\eta_k + \lambda_k)^{-1} \cdot (1 + \lambda_k \alpha^4(1 - \alpha)^{-4})(1 - \gamma)^{-1} \cdot \| \pi_k(\cdot \given s_t) - \pi_{k+1}(\cdot \given s_t)\|_1 \notag \\
& \qquad \leq \kl \bigl( \pi^*(\cdot \given s_t) \| \pi_k(\cdot \given s_t) \bigr) - \kl \bigl( \pi^*(\cdot \given s_t) \| \pi_{k+1}(\cdot \given s_t) \bigr) \\
& \qquad \qquad + (\eta_k + \lambda_k)^{-2} \cdot \bigl (1 + \lambda_k \cdot \alpha^4(1 - \alpha)^{-4} \bigr)^2 \cdot (1 - \gamma)^{-2} \notag, 
\$
which concludes the proof. 
\end{proof}

% =======================================================
% 
%                  Additional Results 
% 
% =======================================================
\section{Supplementary Tables}
\label{app:tables}
\begin{table}[h]
\caption{Hyperparameters of SCORE}
\label{tab:hyperparameters}
\centering
\resizebox{\textwidth}{!}{
\begin{tabular}{lll|lll}
\toprule
\multicolumn{3}{c|}{Basic hyperparameters from TD3~\cite{fujimotoAddressingFunctionApproximation2018}} & \multicolumn{3}{c}{SCORE hyperparameters}  \\ 
Notation & Description & Value                     & Notation & Description & Value             \\
\midrule
$\sigma$ & The std of the Gaussian exploration noise         & 0.2                       & $M$        & The number of critic networks        & 5                 \\
$c$        & The max noise         & 0.5                       & $d_{\text{bc}}$        & The update frequency of the behavior cloning coefficient         & 10000             \\
$d$       & The update frequency of the actor network and the target networks         & 2                         & $\gamma_{\text{bc}}$   & The discount rate of the behavior cloning coefficient         & $\{0.96, 0.98, 1.0\}$   \\
$\tau$     & The target network update rate         & 0.005                     & $\beta$     & The uncertainty penalty coefficient         & $\{0.1, 0.2, 0.5\}$\\   
\bottomrule
\end{tabular}}
\end{table}

\vspace{0.5in}

\begin{table}[h]
\caption{Average normalized scores over 5 random seeds on the D4RL datasets. We compare SCORE with both model-based methods (MOPO, MOReL) and model-free methods (BCQ, BEAR, UWAC, CQL, TD3-BC, PBRL, PBRL w/o prior). The standard deviation is reported in the parentheses. A score of zero corresponds to the performance of the random policy and a score of 100 corresponds to the performance of the expert policy.}
\vskip -0.1in
\label{tab:mujoco}
\begin{center}
\resizebox{\textwidth}{!}{
\begin{tabular}{c@{\hspace{5pt}}lllllllllll}
\toprule
{} & Task & SCORE & MOPO & MOReL & BCQ & BEAR & UWAC & CQL & TD3-BC & PBRL & PBRL w/o prior\\ 
\midrule
\multirow{3}{*}{\rotatebox[origin=c]{90}{Random}} & halfcheetah & 29.1$\pm$2.6 & 35.9$\pm$2.9  & 30.3$\pm$5.9 & 2.2$\pm$0.0 & 2.3$\pm$0.0 & 2.3$\pm$0.0  & 21.7$\pm$0.6 & 10.6$\pm$1.7 & 13.1$\pm$1.2 & 11.0$\pm$5.8\\
{} & hopper      & 31.3$\pm$0.3 & 16.7$\pm$12.2 & 44.8$\pm$4.8 & 8.1$\pm$0.5 & 3.9$\pm$2.3  & 2.6$\pm$0.3  & 8.1$\pm$1.4  & 8.6$\pm$0.4 & 31.6$\pm$0.3 &26.8$\pm$9.3\\
{} & walker2d    & 3.7$\pm$7.0  & 4.2$\pm$5.7   & 17.3$\pm$8.2 & 4.6$\pm$0.7 & 12.8$\pm$10.2  & 1.8$\pm$0.4  & 0.5$\pm$1.3  & 1.5$\pm$1.4 & 8.8$\pm$6.3 & 8.1$\pm$4.4\\
\midrule
\multirow{3}{*}{\rotatebox[origin=c]{90}{\shortstack{Medium\\Replay}}} & halfcheetah & 48.0$\pm$0.7  & 69.2$\pm$1.1 & 31.9$\pm$6.0  & 40.9$\pm$1.1  & 36.3$\pm$3.1  & 36.4$\pm$3.3 & 47.2$\pm$0.4 & 44.8$\pm$0.5 & 49.5$\pm$9.8 & 45.1$\pm$8.0\\
{} & hopper      & 94.0$\pm$1.8 & 32.7$\pm$9.4 & 54.2$\pm$32.0  & 40.9$\pm$16.7 & 52.2$\pm$19.3 & 23.7$\pm$2.6 & 95.6$\pm$2.4 & 57.8$\pm$17.3 & 100.7$\pm$0.4 & 100.6$\pm$1.0\\
{} & walker2d    & 84.8$\pm$1.1  & 73.7$\pm$2.4 & 13.7$\pm$8.0 & 42.5$\pm$13.7 & 7.0$\pm$7.8   & 24.3$\pm$5.4 & 85.3$\pm$2.7 & 81.9$\pm$2.7 & 86.2$\pm$3.4 & 77.7$\pm$14.5\\
\midrule
\multirow{3}{*}{\rotatebox[origin=c]{90}{Medium}} & halfcheetah        & 55.2$\pm$0.4 & 73.1$\pm$2.4  & 20.4$\pm$13.8 & 45.4$\pm$1.7 & 43.0$\pm$0.2 & 42.3$\pm$0.3 & 49.2$\pm$0.3 & 47.8$\pm$0.4 & 58.2$\pm$1.5 & 57.9$\pm$1.5\\
{} & hopper             & 99.6$\pm$2.8 & 38.3$\pm$34.9 & 53.2$\pm$32.1 & 54.0$\pm$3.7 & 51.8$\pm$4.0 & 50.2$\pm$5.2 & 62.7$\pm$3.7 &69.1$\pm$4.5 & 81.6$\pm$14.5 & 75.3$\pm$31.2\\
{} & walker2d           & 89.2$\pm$1.2 & 41.2$\pm$30.8 & 10.3$\pm$8.9 & 74.5$\pm$3.7 & -0.2$\pm$0.1 & 72.8$\pm$4.1 & 83.3$\pm$0.8 & 81.3$\pm$3.0 & 90.3$\pm$1.2 & 89.6$\pm$0.7\\
\midrule
\multirow{3}{*}{\rotatebox[origin=c]{90}{\shortstack{Medium\\Expert}}} & halfcheetah & 92.6$\pm$3.5  & 70.3$\pm$21.9 & 35.9$\pm$19.2 & 94.0$\pm$1.2  & 46.0$\pm$4.7 & 42.8$\pm$0.3 & 70.6$\pm$13.6 & 88.9$\pm$5.3 & 93.6$\pm$2.3 & 92.3$\pm$1.1\\
{} & hopper      & 100.3$\pm$6.9 & 60.6$\pm$32.5 & 52.1$\pm$27.7 & 108.6$\pm$6.0 & 50.6$\pm$25.3 & 48.6$\pm$7.8 & 111.0$\pm$1.2 & 102.0$\pm$10.1 & 111.2$\pm$0.7 & 110.8$\pm$0.8\\
{} & walker2d    & 109.3$\pm$0.5 & 77.4$\pm$27.9 & 3.9$\pm$2.8 & 109.7$\pm$0.6 & 22.1$\pm$44.5 & 96.9$\pm$7.1 & 109.7$\pm$0.3 & 110.5$\pm$0.3 & 109.8$\pm$0.2 & 110.1$\pm$0.3\\
\midrule
\multirow{3}{*}{\rotatebox[origin=c]{90}{Expert}} & halfcheetah        & 96.4$\pm$0.6  & 81.3$\pm$21.8 & 2.2$\pm$5.4 & 92.7$\pm$2.5  & 92.7$\pm$0.6  & 93.5$\pm$0.9 & 97.5$\pm$1.8  & 96.3$\pm$0.9 & 96.2$\pm$2.3 & 92.4$\pm$1.7\\
{} & hopper      & 112.0$\pm$0.3 & 62.5$\pm$29.0 & 26.2$\pm$14.0 & 105.3$\pm$8.1 & 54.6$\pm$21.0 & 103.9$\pm$13.6 & 105.4$\pm$5.9 & 109.5$\pm$4.1 & 110.4$\pm$0.3 & 110.5$\pm$0.4\\
{} & walker2d    & 109.4$\pm$0.6 & 62.4$\pm$3.2  & -0.3$\pm$0.3 & 109.0$\pm$0.4 & 106.8$\pm$6.8 & 108.2$\pm$0.5 & 109.0$\pm$0.4 & 110.3$\pm$0.4 & 108.8$\pm$0.2 & 108.3$\pm$0.3\\
\bottomrule
{} & Overall & \textbf{77.0$\pm$2.0} & 53.3$\pm$16.3 & 26.4$\pm$12.6 & 62.6$\pm$4.0 & 38.8$\pm$10.0 & 50.0$\pm$3.5 & 70.5$\pm$2.5 & 68.1$\pm$3.5 & 76.6$\pm$2.4 & 74.4$\pm$5.3\\
\bottomrule
\end{tabular}}
\end{center}
\end{table}

\end{document}